\documentclass[mnsc,nonblindrev]{informs3aa}

\OneAndAHalfSpacedXI 

\RequirePackage{mathabx}
\RequirePackage[colorlinks,citecolor=blue,urlcolor=blue]{hyperref}

\usepackage{multirow}
\usepackage{enumitem}
\usepackage{epstopdf}

\usepackage{natbib}
 \bibpunct[, ]{(}{)}{,}{a}{}{,}%
 \def\bibfont{\small}%

\usepackage[linesnumbered,boxed,norelsize]{algorithm2e}
\usepackage{blindtext}
\SetKwRepeat{Do}{do}{while}%
\SetKwInput{KwInput}{Input}
\SetKwInput{KwOutput}{Output}
\SetKwInput{KwFunction}{Function}
\SetKwInput{KwParameters}{Parameters}
\SetKwComment{Comment}{$\triangleright$\ }{}
\SetCommentSty{itshape}

\let\oldnl\nl
\newcommand{\nonl}{\renewcommand{\nl}{\let\nl\oldnl}}
\newcommand{\mat}{\mathbf }
\newcommand{\vct}{\boldsymbol }

\newcommand{\kl}{\mathrm{KL}}

\def\mn{\mathfrak n}

\renewcommand{\hat}{\widehat}
\renewcommand{\tilde}{\widetilde}
\renewcommand{\bar}{\overline}

\definecolor{DSgray}{cmyk}{0,1,0,0}
\newcommand{\Authornote}[2]{{\small\textcolor{DSgray}{\sf$<<<${  #1: #2
}$>>>$}}}

\newcommand{\xnote}[1]{{\Authornote{Xi}{#1}}}
\TheoremsNumberedThrough     
\ECRepeatTheorems

\EquationsNumberedThrough    

\MANUSCRIPTNO{}

\begin{document}


\RUNAUTHOR{Chen, Shi, Wang, and Zhou}

\RUNTITLE{Dynamic Assortment Planning Under Nested Logit Models}

\TITLE{Dynamic Assortment Planning Under Nested Logit Models}


\ARTICLEAUTHORS{%
	\AUTHOR{Xi Chen}
	\AFF{Leonard N. Stern School of Business, New York University, \EMAIL{xc13@stern.nyu.edu}}
	\AUTHOR{Chao Shi (corresponding author)}
	\AFF{School of Information Management and Engineering, Shanghai University of Finance and Economics, \EMAIL{shi.chao@sufe.edu.cn}}
	\AUTHOR{Yining Wang}
	\AFF{Warrington College of Business,
		University of Florida, \EMAIL{yining.wang@ warrington.ufl.edu}}
	\AUTHOR{Yuan Zhou }
	\AFF{   Department of Industrial and Enterprise Systems Engineering, 	Department of Computer Science (Affiliate), University of Illinois at Urbana-Champaign, \EMAIL{yuanz@illinois.edu}}
} 

\ABSTRACT{%
We study a stylized dynamic assortment planning problem during a selling season of finite length $T$. At each time period, the seller offers an arriving customer an assortment of  substitutable products and the customer makes the purchase among offered products according to a discrete choice model. The goal of the seller is to maximize the expected revenue, or equivalently, to minimize the worst-case expected regret.  One key challenge  is that utilities of products are unknown to the seller and need to be learned. Although the dynamic assortment planning problem has received increasing attention in revenue management, most existing work is based on the multinomial logit choice models (MNL). In this paper, we study the problem of dynamic assortment planning under a more general choice model---the nested logit model, which models hierarchical choice behavior and is  ``the most widely used member of the GEV (generalized extreme value) family'' \citep{Train2009}. By leveraging the revenue-ordered structure of the optimal assortment within each nest, we develop a novel upper confidence bound (UCB) policy with an aggregated estimation scheme. Our policy simultaneously learns customers' choice behavior and makes dynamic decisions on assortments based on the current knowledge. It achieves the accumulated regret at the order of $\tilde{O}(\sqrt{MNT})$,  where $M$ is the number of nests and $N$ is the number of products in each nest. We further provide a lower bound result of $\Omega(\sqrt{MT})$, { which shows the near optimality of the upper bound when $T$ is much larger than $M$ and $N$.  When the number of items per nest $N$ is large, we further provide a discretization heuristic for
better performance of our algorithm.} Numerical results are presented to demonstrate the empirical performance of our proposed algorithms.
}


 \HISTORY{Received: August 2019; accepted: July 2020 by Dan Zhang after two rounds of revision}

\maketitle

%


\section{Introduction}

Assortment planning has a wide range of applications in retailing and online advertising. Given a large number of substitutable products,
the assortment planning problem refers to the selection of a subset of products (a.k.a., an assortment) offered to a customer such that the expected revenue is maximized. To model customers' choice behavior when facing a set of offered products, discrete choice models, which  capture demand for each product as a function of the entire assortment, have been widely used. One of the most popular discrete choice models is the \emph{multinomial logit model (MNL)}, which naturally results from the random utility theory where a customer's preference of a product is represented by the mean utility of the product with a random factor \citep{McFadden1974}. An important extension of the MNL is the nested logit model \citep{Williams77,McFadden1980,Borch-Supan1990} that models a customer's choice in a hierarchical way: a customer first selects a category of products (known as a nest), and then a product within the category. When the mean utilities of the products are given, the static assortment optimization problem under MNL or nested logit models can be efficiently solved \citep{Talluri2004,davis2014assortment}.

In many scenarios, customers' choice behavior (e.g., mean utilities of products) is not given as \emph{a priori} and cannot be easily estimated due to the insufficiency of historical data (e.g.,  fast fashion sale or online advertising). To address this challenge,  dynamic assortment planning that simultaneously learns choice behavior and makes decisions about the assortment has received a lot of attention \citep{Caro2007, Rusmevichientong2010, Saure2013, Agrawal16MNLBandit, Agrawal17Thompson, Chen:18tight, Chen:18near}. More specifically, in a dynamic assortment planning problem, the seller offers an assortment (or a set of assortments for different nests in a nested logit model) to each arriving customer in a finite time horizon $T$, observes the purchase behavior of the customer, and then updates the learned information about the underlying demand function. The goal of the seller is to maximize the cumulative expected revenue over $T$ periods. In the literature,
the \emph{regret} is often adopted to measure the performance of a given dynamic assortment planning policy, which is defined as the gap between the  expected revenue generated by the policy and the oracle expected revenue when the mean utility for each product is known as \emph{a priori}.

In existing dynamic assortment literature, the underlying choice model is usually assumed to be an MNL model \citep{Rusmevichientong2010, Saure2013, Agrawal16MNLBandit, Agrawal17Thompson,  Chen:18near}.
(The work of \citep{Saure2013} also considered other forms of choice models, in addition to the MNL model.)
 In this paper, we study this problem under a more general choice model---the two-level nested logit model. Indeed, the nested logit model is considered as ``the most widely used member of the GEV (generalized extreme value) family'' and ``has been applied by many researchers in a variety of situations'' (see Chapter 4 from \cite{Train2009}).  It is well known that the standard MNL suffers from the independence of irrelevant alternatives (IIA), which implies proportional substitution across alternatives (see Chapter 4 from \cite{Train2009}). The nested logit model relaxes the IIA assumption on alternatives in different nests and thus provides a richer set of substitution patterns. Despite the importance of the nested logit model, the dynamic assortment planning question
 under nested logit models remains an open problem in revenue management due to the complicated structure of nested logit models.


The main contribution of this paper is to develop computationally efficient policies for addressing this problem. Assume that there are $M$ nests and each nest has $N$ possible products to recommend. By leveraging the revenue-ordered structure of optimal assortments and the idea of aggregate estimation of next-level utilities, 
we propose the first upper confidence bound (UCB)-based policy, which leads to a non-asymptotic regret bound, {in which the dominating term involving $T$ is $\tilde{O}(\sqrt{MNT})$} (see Corollary \ref{cor:regret-preliminary} for a more precise bound). Here, $\tilde{O}$ hides the logarithmic dependence on $T, N$, and $M$.

Our second contribution is to understand the information-theoretical limitation of the problem. In particular,  we further provide a lower bound on the regret $\Omega(\sqrt{MT})$ (see Theorem \ref{thm:lower-bound}). {  First, this lower bound shows that when the time horizon $T$ is sufficiently large, our upper bound is within a factor of $\sqrt{N}$ of the lower bound, where $N$ is the number products within each nest and is smaller than the total number of products. The optimal dependence on $N$ is, however, a technically very challenging question and is beyond the scope of this paper. Nevertheless, for the case of $N$ being large, we introduce a discretization technique, which provides a useful heuristic leading to a much improved dependence on $N$. Through simulation studies, we found the discretization heuristic to be very effective with improved performance when there are many items per nest.
} Second, this lower bound also demonstrates a fundamental difference between the nested logit models and standard (plain) MNL models. According to \cite{Chen:18near},
the standard MNL admits a tight lower bound of $\Omega(\sqrt{T})$, independent of other problem parameters (e.g., the number of products). In contrast, for nested logit models, our lower bound shows that, in addition to $\sqrt{T}$, dependency on the number of nests $M$ is unavoidable.

The details of the proposed policies will be presented in the main paper and here we briefly highlight the key technical points  in the proposed policies:
\begin{enumerate}
\item Leveraging the revenue-ordered structure: For $N$ products in each of the $M$ nests, the total number of possible assortment combinations (i.e., the size of the action space) will be $(2^N)^M$, which is exponentially large. By leveraging the revenue-ordered structure of the optimal assortment within each nest (see Lemma \ref{lem:nested-popular} and \cite{davis2014assortment,Li2015d}), the size of the action space can be effectively reduced to $O(N^M)$.  However, the size of this reduced action space is still too large if one directly applies existing bandit learning algorithms that treat each assortment in the action space independently, which will incur a regret related to $N^M$. To address this challenge, we propose an aggregate estimation technique as follows.

\item Aggregate estimation:  A key point of the paper is that estimating utility parameters for each individual product (that will incur a large regret) is unnecessary for dynamic assortment planning. Instead, we propose an aggregate estimation technique that only estimates the preference and revenue parameters on a nest level.
More specifically, in our algorithm only \emph{level sets} of assortments within each nest are considered,
which have both \emph{unknown} aggregated revenue and utility parameters.

Another advantage of our algorithm is that it shows that estimation of exponent parameters $\{\gamma_i\}_{i=1}^M$ (see Eq.~(\ref{eq:vi}) in the nested logit model specification) is not necessary.
 Instead, we directly estimate ``nested-level utility'' $V_i(\cdot)^{\gamma_i}$ (see Eq. \eqref{eq:simpler-model} and the discussions above Eq. \eqref{eq:simpler-model}).


\item UCB policy: 
We propose an upper confidence bound (UCB) algorithm using an epoch-based strategy from \cite{Agrawal16MNLBandit},
which leads to a worst-case expected regret of $\tilde{O}(\sqrt{MNT})$.   Although the UCB  has been a well-known technique for bandit problems, adopting this high-level idea to solve a problem with specific structures certainly requires technical innovations (e.g., how to build a confidence bound on a carefully designed parameter, see Lemma \ref{lem:new-nested-concentration}).
 {    We further note that our UCB policy generalizes the one in \cite{Agrawal16MNLBandit} because in our model the ``level sets'' are constructed within each nest,
and therefore both their revenue and utility parameters are unknown (see Eq. \eqref{eq:nest-level-utility} in Sec.~\ref{sec:reduction})
This contrasts the setting in \cite{Agrawal16MNLBandit} in which the revenue parameters of each single item are known.}

\item  Discretization technique:  When $N$ is large, we introduce a discretization technique to reduce the size of the action space to $O((1/\delta)^M)$, where $\delta$ is discretization granularity. Our policy without discretization corresponds to a special case of $\delta=0$.    We are able to show that the proposed space reduction techniques lose very little in terms of optimal expected revenue, i.e., the gap of the optimal expected revenue between all the possible assortment combinations and the  reduced action space is at most $\delta$ (see Lemma \ref{lem:discretization}).
Simulation studies confirm the effectiveness of this discretization  heuristic.



\end{enumerate}

To the best of our knowledge, our policies are the first policies for dynamic assortment planning under the nested logit model,
which presents unique challenges compared to the standard MNL model as the nest-level revenues of assortment selections are not known and have to be estimated on the fly.  It is also worthwhile noting that due to the complicated structure of the nested logit model, it is technically challenging to derive a tight lower bound on the regret in terms of $N$, {and  we suspect that the current lower bound $\Omega(\sqrt{MT})$ is not tight and misses a $\sqrt{N}$ factor (see more detailed discussions in Remark \ref{rem:lowerbound} in Sec.~ \ref{sec:lower})}. Since the main focus of the paper is to derive the first efficient policy for dynamic assortment planning under nested logit models, we leave this challenging technical problem for future works.


\subsection{Related Works}

Static assortment planning with known choice behavior has been an active research area since the seminal works by \citet{Ryzin1999} and \citet{Mahajan2001}. When the customer makes the choice according to the MNL model, \citet{Talluri2004} and \citet{Gallego2004} proved an optimal  assortment will belong to revenue-ordered assortments (a.k.a. nested-by-revenue assortments). An alternative proof is provided in \cite{Liu2008}. This important structural result enables the efficient computation of static assortment planning under the MNL model, which reduces the number of candidate assortments from $2^N$ to $N$, where $N$ is the number of products per nest. When there is a set constraint on the assortment set, an efficient polynomial-time algorithm (with running time $O(N^2)$) was proposed in \cite{Rusmevichientong2010}. For nested logit models, \citet{davis2014assortment} proved an important structural result that the optimal assortment within each nest is revenue-ordered, which will also be used in designing our dynamic policies. Assuming that there are $M$ nests and $N$ products within each nest, \cite{Li2014} further proposed an efficient greedy algorithm to find an optimal assortment set with $O(N M\log M)$ time complexity. \citet{Kok11} considered the joint assortment optimization and pricing problem with a restricted number of nests.
There are several recent works on static assortment planning under variants of nested logit models. For example, \citet{Gallego2014} studied the constrained nested logit model;  \citet{Li2015d} extended the popular two-level nested logit model to a $d$-level nested logit model with $d\geq 2$; \citet{Zhang17paired} studied the paired combinatorial nested logit model. In addition, there are extensive research on static assortment optimization for more complex choice models, e.g., a robust version of MNL  \citep{rusmevichientong2012robust},  the mixture of logit models \citep{Bront2009, Mendez-Diaz2014,rusmevichientong2014assortment}, Markov chain-based choice models \citep{Blanchet2016,Desir2015}, the  generalized attraction model \citep{Wang:13:assortment}, Mallows-based choice models \citep{Desir2016}, a multiple attempt model \citep{Chung:18}, contextual MNL \citep{Wang:17:person}, and a general class of choice models based on a distribution over permutations \citep{Farias2013}.

{
\cite{davis2014assortment} considered the nested logit model studied in this paper, with
both the cases of $\gamma_i\leq 1$ and $\gamma_i>1$.
 In the case of $\gamma_i\leq 1$, they established the revenue-order
property within each nest, but considered an alternative linear programming type algorithm
to solve for optimal assortments efficiently. 
 In \cite{Li2015d} the optimization question of
$d$-level nested logit models is considered, and efficient fractional programming based methods
are developed. Our optimization subroutine (see Sec.~\ref{subsec:optimization}) turns out to be similar as the one in \cite{Li2015d} in the
 special case of $d=2$, with the difference being that upper-confidence estimates of $R_i(S_i)$ and
$V_i(S_i)^{\gamma_i}$ are used in our optimization, while in \cite{Li2015d} the full-information parameter
values were used.
}

Due to increasing popularity of data-driven revenue management, researchers have started to relax the assumption about fully available prior knowledge of customers' choice behavior and investigate dynamic assortment planning. Motivated by fast-fashion retailing, the work by \cite{Caro2007} was among the first to study the dynamic assortment planning problem, which assumes that the demand for products is independent of each other. \cite{Bertsimas:17} studied a two-step problem with separate demand estimation and assortment planning, where the first step estimates a generic ranking-based choice model and the second step solves a mixed-integer optimization for assortment planning. \cite{Rusmevichientong2010, Saure2013}, \cite{Agrawal16MNLBandit, Agrawal17Thompson}, and \cite{Chen:18near} incorporated choice models of MNL into dynamic assortment planning, formulating the problem into an online regret minimization problem. 
{ However, the extension of the plain MNL model to nested logit models is highly nontrivial and requires several technical innovations. For example, instead of estimating utility parameters for each product, we estimate nest-level aggregated quantities (see more discussions in the introduction).  Furthermore, we introduce a discretization technique to alleviate the effect of having many items per nest.} 

There is another line of recent research on investigating the assortment planning question in which each arriving customer could have a different choice behavior. For example, \cite{Golrezaei2014} and \cite{Chen:16recom} assumed that each customer's choice behavior is known but that the customers' arriving sequence can be adversarially chosen, and took into account both the revenue and inventory levels. Since the arriving sequence can be arbitrary, there is no learning component in the problem and both \cite{Golrezaei2014} and \cite{Chen:16recom} adopted the competitive ratio as the performance evaluation metric.  In addition, there are a few recent works studying joint assortment planning and pricing under MNL models (see e.g., \cite{Wang:12:joint}, \cite{Besbes:14}, and \cite{Miao:18}). It would also be an interesting future work to consider dynamic joint assortment planning and pricing under nested MNL models.

\subsection{Notations and paper organizations}

Throughout the paper, we use $f(\cdot) \lesssim g(\cdot)$ to denote that $f(\cdot) = O(g(\cdot))$,
or more specifically $\limsup_{T\to\infty}|f(T)|/|g(T)|<\infty$.
 Similarly, by $f(\cdot) \gtrsim g(\cdot)$, we denote $f(\cdot) = \Omega(g(\cdot))$. We also use $f(\cdot) \asymp g(\cdot)$ for $f(\cdot) = \Theta(g(\cdot))$. In the paper, $\tilde{O}(\cdot)$ is used to hide logarithmic factors on $T$, $N$, and $M$. The rest of the paper is organized as follows: In Section \ref{sec:app_nested}, we first provide the background of  nested logit models and introduce an important structural result on optimal assortments \citep{davis2014assortment, Li2015d}.
In Section \ref{sec:nest_LUCB}, we propose our UCB policy and establish the corresponding regret bound.
A lower bound on regret is provided in Section \ref{sec:lower}.
 The numerical results are provided in Section \ref{sec:numerical}, followed by the conclusion in Section \ref{sec:conclusions}.

\section{Model specifications and assortment space reductions}
\label{sec:app_nested}

In this section we formally introduce the nested logit assortment choice model considered in this paper.
We restrict ourselves to two-level nested logit models, where items are organized as $M$ known commodity \emph{nests}
and customers' purchasing actions are modeled by a \emph{hierarchical} multinomial logit model (more details given in Section~\ref{subsec:model-nested}).

\subsection{The nested logit model}\label{subsec:model-nested}

We use $[M]=\{1,2,\cdots,M\}$ to label the $M$  nests.
For each nest $i\in[M]$, label the items in nest $i$ by $[N_i]=\{1,2,\cdots,N_i\}$.
Each item $j\in [N_i]$
is associated with a \emph{known} revenue parameter $r_{ij}$ and an \emph{unknown} mean utility parameter $v_{ij}$.
We assume each nest has an equal number of items, i.e., $N_1=\cdots=N_M=N$.
Further, let $\{\gamma_i\}_{i\in[M]}\subseteq [0,1]$ be a collection of \emph{unknown} correlation parameters for different nests. Each parameter $\gamma_i$ is a measure of the degree of independence among the items in nest $i$: a larger value of $\gamma_i$ indicates less correlation (see Chapter 4 of \cite{Train2009}).

At each time period $t\in\{1,2,\cdots,T\}$, the retailer offers the arriving customer an assortment $S_i^{(t)}\in \mathbb S_i=2^{[N]}$ for every nest $i\in[M]$, conveniently denoted as $\mat S^{(t)}=(S_1^{(t)},\cdots,S_M^{(t)})$.
The retailer then observes a nest-level purchase option $i_t\in [M]\cup\{0\}$. If  $i_t\in[M]$, an item $j_t\in[N]$ is purchased within the nest $i_t$. On the other hand, $i_t=0$ means no purchase occurs at time $t$. The probabilistic model for the purchasing option $(i_t,j_t)$ can be formulated as below:
\begin{equation}
\Pr\left[i_t=i|\mat S^{(t)}\right] = \frac{V_i(S_i^{(t)})^{\gamma_i}}{V_0 + \sum_{i'=1}^MV_{i'}(S_{i'}^{(t)})^{\gamma_{i'}}},\quad \text{where}\;\;
V_0\equiv 1, \;\;  V_{i}(S_{i}^{(t)}) = \sum_{j\in S_{i}^{(t)}}v_{ij}\;\;\text{for $i\in[M]$};
\label{eq:vi}
\end{equation}
\begin{equation}
\Pr\left[j_t=j|i_t=i,\mat S^{(t)}\right] = \frac{v_{ij}}{\sum_{j'\in S_i^{(t)}}v_{ij'}} \;\;\;\;\text{for}\;\; i\in[M], \;\; j\in S_i^{(t)}.
\end{equation}
{ Note that when $\gamma_i=1$ for all $i \in[M]$, the nested logit model reduces to the standard MNL model.}

The retailer then collects revenue $r_{i_t,j_t}$ provided that $i_t\neq 0$.
The expected revenue $R(\mat S^{(t)})$ given the assortment combination $\mat S^{(t)}$ can then be written as
{ \begin{eqnarray}
		R(\mat S^{(t)})& =& \sum_{i=1}^M \Pr\left[i_t=i|\mat S^{(t)}\right] \sum_{j \in S_i^{(t)}} r_{ij}\Pr\left[j_t=j|i_t=i,\mat S^{(t)}\right] \nonumber\\
		& = & \frac{\sum_{i=1}^MR_i(S_i^{(t)})V_i(S_i^{(t)})^{\gamma_i}}{1 + \sum_{i=1}^MV_i(S_i^{(t)})^{\gamma_i}};
		\qquad \qquad \text{where}\;\; R_i(S_i^{(t)}) = \frac{\sum_{j\in S_i^{(t)}}r_{ij}v_{ij}}{\sum_{j\in S_i^{(t)}} v_{ij}}.
		\label{eq:ri}
\end{eqnarray}}
The objective of the seller is to minimize \emph{expected (accumulated) regret}, defined as follows:
\begin{equation}
\mathrm{Regret}(\{\mat S^{(t)}\}_{t=1}^T) := \sum_{t=1}^T R^* - \mathbb E\left[R(\mat S^{(t)})\right],\;\;\;\;\text{where}\;\;\;\; R^* =\max_{\mat S\in\mathbb S=\mathbb S_1\times\cdots\times\mathbb S_M} R(\mat S).
\label{eq:nested-regret}
\end{equation}

Throughout the paper, we make the following boundedness assumptions on revenue and utility parameters:
\begin{enumerate}[leftmargin=0.5in]
	\item[(A1)] The revenue parameters satisfy $0\leq r_{ij} \leq 1$ for all $i \in [M]$ and $j \in [N]$.
	\item[(A2)]  The utility parameters satisfy  $0<v_{ij}\leq C_V$ for all $i \in [M]$ and $j \in [N]$ with some constant $C_V \geq 1$.
\end{enumerate}
The first boundedness assumption on revenue parameters is standard in the literature (see e.g.,   Theorem 1 in \cite{Agrawal16MNLBandit}). It is also worthwhile noting that assumption (A2) is weaker than the common assumption that no purchase (with $V_0=1$) is the most frequent outcome.
{ Both assumptions can be regarded as without loss of generality as the parameter values could be normalized.}

{
We remark that in the original nested-logit model assortment planning paper \citep{davis2014assortment},
it is allowed that $\gamma_i>1$ and furthermore there is a no-purchase option \emph{within each nest}.
We assumed $\gamma_i\leq 1$ because it is the setting in which the full-information combinatorial optimization problem
is easy to solve, which is the foundation of our theoretical regret analysis.
Indeed, when $\gamma_i$ exceeds one, it is proved in the work of Davis et al.~(2014) that the combinatorial optimization question
(when all parameters are known) is NP-hard, and only approximation algorithms can be developed.

We do not allow for a no-purchase option within each nest, on the other hand, for a more technical reason.
In our proposed learning-while-doing algorithm, it is critical to count the number of times that each nest $i$ is selected by customers
until a no-purchase action \emph{on the nest level} occurs.
If we allow for no-purchase options within each nest, our algorithm will no longer be able to distinguish between the events of no-purchase
on the nest level or within nests. This leads to biased estimates of $V_i(S_i)^{\gamma_i}$ parameters and potentially linear regret.
Hence, we choose not to include no-purchase options within nests for a cleaner algorithm and analysis.
}





\subsection{Assortment space reductions}
\label{sec:reduction}

For nested logit models, the complete assortment selection space (a.k.a. action space) $\mathbb S=\mathbb S_1\times \mathbb S_2\times\cdots\mathbb \times \mathbb S_M$ is extremely large, consisting of an exponential number of candidate assortment selections  { (on the order of $(2^N)^M$)}.
Existing bandit learning approaches treating each assortment set in $\mathbb S$ independently would easily incur a regret also exponentially large. It is thus mandatory to reduce the number of candidate assortment sets in $\mathbb S$.

Fortunately, existing results on the structure of optimal $\mat S$ show that it suffices
to consider \emph{level sets} $\mathcal L_i(\theta_i) := \{j\in[N]: r_{ij}\geq\theta_i\}$ for each nest $i$. { In other words, $\mathcal L_i(\theta_i)$ is the set of products in nest $i$ with revenue larger than or equal to a given threshold $\theta_i \geq 0$}.
{Define $\mathbb P_i := \left\{\mathcal L_i(\theta_i): \theta_i\geq 0\right\}\subseteq \mathbb S_i$ to be all the possible level sets of $\mathbb S_i$  and let
\begin{equation}\label{eq:P_space}
\mathbb P := \mathbb P_1\times\mathbb P_2\times\cdots\times\mathbb P_M\subseteq\mathbb S.
\end{equation}
}

The following lemma from \cite{davis2014assortment} and \cite{Li2015d} shows that one can restrict the assortment selections to $\mathbb P$ without loss of any optimality in terms of expected revenue.
\begin{lemma}[\cite{davis2014assortment,Li2015d}]
	There exists level set threshold parameters $(\theta_1^*, \ldots, \theta_M^*)$ and $\mat S^* = (\mathcal L_1(\theta_1^*),\cdots,\mathcal L_M(\theta_M^*))\in\mathbb P$ such that the following hold:
	\begin{enumerate}
		\item $R(\mat S^*) = \max_{\mat S\in\mathbb S}R(\mat S) = R^*$;
		\item $\theta_i^* \geq \gamma_iR^* + (1-\gamma_i)R_i(S_i^*)$ for all $i\in[M]$, where $S_i^*=\mathcal L_i(\theta_i^*)$.
	\end{enumerate}
	\label{lem:nested-popular}
\end{lemma}
The first item in Lemma \ref{lem:nested-popular} is an important structural result showing that the optimal assortments are ``revenue-ordered'' within each nest. The second item is a technical result, which will be used in the proof.
Compared to the original action space $\mathbb S$, the reduced ``level set'' space $\mathbb P$ is much smaller,
with each $\mathbb P_i$ consisting of $N$ instead of $2^N$ candidate assortments.

With Lemma \ref{lem:nested-popular}, an assortment combination $\mat S=(S_1,\cdots,S_M)\in\mathbb P$ can then be parameterized by a vector $\vct\theta=(\theta_1,\cdots,\theta_M)\in([0,1]\cup\{\infty\})^M$,
such that $\mat S(\vct\theta) = (\mathcal L_1(\theta_1),\cdots,\mathcal L_M(\theta_M))$.
Note that $\mathcal L_i(\infty)=\emptyset$ indicates the empty set for nest $i$.
Denote $\mathcal K_i = [0,1]\cup\{\infty\}$, and for any $i\in[M]$, $\theta_i\in\mathcal K_i$ define
\begin{equation}\label{eq:nest-level-utility}
u_{i,\theta_i}:= V_i(\mathcal L_i(\theta_i))^{\gamma_i} \quad \text{and} \quad \phi_{i,\theta} := R_i(\mathcal L_i(\theta_i)),
\end{equation}
where $V_i(\cdot)$ and $R_i(\cdot)$ are nest-level utility parameter and expected revenue associated with the level set $\mathcal L_i(\theta_i)$
 (see definitions of $V_i$ and $R_i$ in Eq \eqref{eq:vi} and \eqref{eq:ri}, respectively). We note that it is fundamentally different from the standard MNL: the nest-level expected revenue $\phi_{i,\theta}$, which depends on utility parameters, is unknown and needs to be learned; while the revenue of each product in a standard MNL is known to the seller prior to the first selling period.
By our assumptions (A1) and (A2), it is easy to verify that $\phi_{i,\theta_i}\in[0,1]$ and $u_{i,\theta_i}\in [0, (NC_V)^{\gamma_i}]\subseteq [0,NC_V]$ for all $i \in [M]$ and $\theta_i \in \mathcal K_i$.
We also note that $(NC_V)^{\gamma_i}\leq NC_V$, since $\gamma_i\in[0,1]$ and $C_V\geq 1$.
{
Furthermore, because each nest consists of at most $N$ products, the sets $\mathcal K_i$ can be made \emph{finite}
by considering only levels $\theta_i$ corresponding to revenue parameters of the $N$ products.
}

Let $i_t\in[M]\cup\{0\}$ be the nest the customer selects at time $t$ and $r_t$ be the collected revenue.
The expected revenue for assortment $\mat S^{(t)}$ parameterized by $\vct\theta^{(t)}$ can then be expressed as
\begin{equation}
\Pr[i_t=i|\vct \theta^{(t)}] = \frac{u_{i,\theta_i}}{1+\sum_{i'=1}^Mu_{i',\theta_{i'}}}; \;\;\;\;\;\; \mathbb E[r_t|i_t=i] = \phi_{i,\theta_i}; \;\; r_t=0\;\;a.s.\;\;\text{if }i_t=0.
\label{eq:simpler-model}
\end{equation}
{Therefore, the expected revenue for an assortment combination  parameterized by $\vct\theta^{(t)}$ takes the following form:
\[
R'(\vct \theta^{(t)}) := \sum_{i=1}^M\Pr[i_t=i|\vct \theta^{(t)}]\cdot \mathbb E[r_t|i_t=i] =  \frac{{\sum_{i=1}^M\phi_{i,\theta_i}u_{i,\theta_i}}}{{1+\sum_{i=1}^Mu_{i,\theta_i}}},
\]
and the regret in \eqref{eq:nested-regret} can be equivalently written as,
\begin{equation}
\mathrm{Regret}(\{\vct\theta^{(t)}\}_{t=1}^T) := \mathbb E\sum_{t=1}^T R'(\vct\theta^*)-R'(\vct\theta^{(t)})\;\;\;\;\;\;\text{where}\;\; R'(\vct\theta^*)=\max_{\vct\theta \in \mathcal K_1\times\cdots\times\mathcal K_M}R'(\vct\theta).
\label{eq:simpler-regret}
\end{equation}
}

%

\section{UCB-based dynamic assortment planning policies}
\label{sec:nest_LUCB}

\setcounter{AlgoLine}{0}
\begin{algorithm}[t]
	\KwInput{Parameter space of $\vct\theta$: $\mathcal K_1,\cdots,\mathcal K_M$, upper bound $U$ on $\{u_{i,\theta}\}$ in \eqref{eq:nest-level-utility}.}
	\KwOutput{assortment sequences $\vct\theta^{(1)},\cdots,\vct\theta^{(T)}\in\mathcal K_1\times\cdots\times\mathcal K_M$.}
	
	
	Initialization: $\tau = 1$, $\{\mathcal E_\tau\}_{\tau=1}^{\infty} = \emptyset$, $t=1$; 
	{for every $i\in[M]$ and $\theta\in\mathcal K_i$,  set $\mathcal T(i,\theta)=\emptyset$, $T(i,\theta)= 0$, $\hat\phi_{i,\theta}=\overline\phi_{i,\theta}=1$, $\hat u_{i,\theta}=\overline u_{i,\theta}=U$;
		for all $i\in[M]$ and $\theta\in\mathcal K_i$ corresponding to the empty assortment (i.e., $\mathcal L_i(\theta)=\emptyset$),
		set $\overline\phi_{i,\theta}=\phi_{i,\theta}=\overline u_{i,\theta}=u_{i,\theta}=0$}\;
	%
	
	\While{$t\leq T$}{
		Find $\hat{\vct\theta}^{(\tau)}=\hat{\vct\theta}\gets \arg\max_{\vct\theta\in\mathcal K_1\times\cdots\times\mathcal K_M} \bar R'(\vct\theta)$, where $\bar R'(\vct\theta) = [\sum_{i=1}^M\bar\phi_{i,\theta_i}\bar u_{i,\theta_i}]/[1+\sum_{i=1}^M\bar u_{i,\theta_i}]$;\label{line:alg-nested-improved-max} \\
		\Comment{\small This optimization problem can be solved in polynomial time; see Sec.~\ref{subsec:optimization};}
		
		
		\Repeat{$i_{t-1} = 0$ or $t > T$}{\label{line:alg-nested-improved-10}
			Pick $\vct\theta^{(t)} = \hat{\vct\theta}$ and observe $i_t,r_t$ in Eq.~(\ref{eq:simpler-model})
			and update $\mathcal E_\tau \gets \mathcal E_\tau \cup\{t\}$, $t=t+1$\;
		}\label{line:alg-nested-improved-14}
		\For{each $i\in[M]$ with $\mathcal L_i(\hat\theta_i)\neq\emptyset$}{
			Compute $\hat n_{i,\tau} = \sum_{t'\in\mathcal E_\tau}\mathbb I[i_{t'}=i]$ and $\hat r_{i,\tau}=\sum_{t'\in\mathcal E_\tau}r_{t'}\mathbb I[i_{t'}=i]$;
			
			Let $\theta=\hat \theta_i$ (for notational simplicity) and update:
			$
			\mathcal T(i, \theta)\gets\mathcal T(i, \theta)\cup\{\tau\}, \quad T(i, \theta) \gets T(i, \theta)+1;
			$
			
			
			Update the utility and mean revenue estimates and as well as their associated confidence bounds:
			$\hat u_{i,\theta}=\frac{1}{T(i,\theta)}\sum_{\tau'\in\mathcal T(i,\theta)}\hat n_{i,\tau'}, \quad
			\hat\phi_{i,\theta}=\frac{\sum_{\tau'\in\mathcal T(i,\theta)}\hat r_{i,\tau'}}{\sum_{\tau'\in\mathcal T(i,\theta)\hat n_{i,\tau'}}}$;
			
			\eIf{$T(i,\theta)\geq 96\ln(2MTK)$}{
				$\bar u_{i,\theta}=\min\{U, \hat u_{i,\theta}+ \sqrt{\frac{96\max(\hat u_{i,\theta},\hat u_{i,\theta}^2)\ln(2MTK)}{T(i,\theta)}} +\frac{144\ln(2MTK)}{T(i,\theta)}\}$, $\bar\phi_{i,\theta} = \min\{1, \hat\phi_{i,\theta}+\sqrt{\frac{\ln(2MTK)}{T(i,\theta)\hat u_{i,\theta}}}\};
				$
			}
			{
				$\bar u_{i,\theta} = U, \quad \bar \phi_{i,\theta}=1$;
			}
		}
		$\tau \gets\tau + 1$\;
	}

	%
	%
	%
	\caption{The upper confidence bound (UCB) policy for dynamic assortment planning.}
	\label{alg:ucb}
\end{algorithm}

In this section we design dynamic planning policies under the nested logit model using an \emph{upper-confidence-bound (UCB)} approach. 
%
The details and pseudo-code of our proposed policy are given in Algorithm \ref{alg:ucb}.

{
The high-level idea behind Algorithm \ref{alg:ucb} is as follows:
for every nest $i$ and level set $\theta\in\mathcal K_i$, a pair of upper confidence estimates $\bar\phi_{i,\theta}$ and $\bar u_{i,\theta}$ are
constructed and maintained, estimating the nest-level revenue and utility parameters
$R_i(\mathcal L_i(\theta))=(\sum_{j\in \mathcal L_i(\theta)}r_{ij}v_{ij})/(\sum_{j\in\mathcal L_i(\theta)}v_{ij})$,
$V_i(\mathcal L_i(\theta))^{\gamma_i} = (\sum_{j\in\mathcal L_i(\theta)}v_{ij})^{\gamma_i}$.
For every potential customer, an optimal assortment combination based on current (upper) parameter estimates $\bar\phi_{i,\theta},\bar u_{i,\theta}$
are computed, which is then offered to the customers repetitively until a no-purchase action occurs.
Afterwards, the parameter estimates $\bar\phi_{i,\theta},\bar u_{i,\theta}$ are updated for all assortments provided in each nest,
and the dynamic assortment planning procedure continues until a total of $T$ customers are served.
}

We next explain a few notations used in the algorithm and then describe the details of the algorithm. The proofs of the results in this section are provided in the supplementary material.
\begin{itemize} 
	\item[-] $\mathcal E_\tau$: all iterations in epoch $\tau$ where the same assortment combination $\vct{\theta}$ is provided. {Each epoch (corresponding to Steps~\ref{line:alg-nested-improved-10}-\ref{line:alg-nested-improved-14} in Algorithm \ref{alg:ucb}) terminates whenever the no-purchase action is observed.} In other words,  one and only one ``no-purchase'' action $i_t=0$ appears at the last iteration of each epoch $\mathcal E_\tau$.
	
	\item[-] $\mathcal T(i,\theta)$: the indices of epochs in which $\theta\in\mathcal K_i$ is supplied in nest $i$;
	$T(i,\theta)=|\mathcal T(i,\theta)|$ denotes the cardinality of $\mathcal T(i,\theta)$;

	\item[-] $\hat n_{i,\tau}$: the number of iterations in the epoch $\tau$ (i.e., $\mathcal E_\tau$) in which an item in nest $i$ is purchased;
	\item[-] $\hat r_{i,\tau}$: the total revenue collected for all iterations in $\mathcal E_\tau$ in which an item in nest $i$ is purchased;
	\item[-] $\hat u_{i,\theta}, \hat\phi_{i,\theta}, \bar u_{i,\theta}, \bar\phi_{i,\theta}$: estimates of $u_{i,\theta},\phi_{i,\theta}$, and their upper confidence bounds.
\end{itemize}

{The epoch-based strategy (i.e., offering the same assortment until no-purchase is observed) in Algorithm \ref{alg:ucb}
 was first introduced by \cite{Agrawal16MNLBandit} and enjoys the favorable properties stated in the next lemma.
\begin{lemma}
		For each epoch $\mathcal E_\tau$ and nest $i\in [M]$, let $\hat\theta_i\in\mathcal K_i$ be such that assortment $\mathcal L_i(\hat\theta_i)$ is provided in nest $i$ in epoch $\tau$.
		The expectations of the number of iterations and total revenues collected in which nest $i$ is purchased (denoted by $\hat n_{i,\tau}$ and $\hat r_{i,\tau}$, respectively, in Algorithm \ref{alg:ucb}) satisfy
		the following regardless of the other offered assortments $\hat\theta_{i'}$ for $i'\neq i$ in the same epoch:
		\begin{enumerate}
			\item $\mathbb E[\hat n_{i,\tau}] = u_{i,\hat\theta_i}$;
			\item $\mathbb E[\hat r_{i,\tau}|\hat n_{i,\tau}] = \hat n_{i,\tau}\phi_{i,\hat\theta_i}$.
		\end{enumerate}
		\label{lem:negative-binomial}
\end{lemma}
	
	
	The above properties motivate intuitive parameter estimators $\hat u_{i,\theta},\hat \phi_{i,\theta}$ of $u_{i,\theta}$ and $\phi_{i,\theta}$ for $\theta=\hat\theta_i$,
	which are taken to be the sample averages of $\hat n_{i,\tau}$ and $\hat r_{i,\tau}$ over all prior epochs $\mathcal E_\tau$ in which the assortment corresponding to level set $\mathcal L_i(\hat\theta_i)$ in nest $i$ is offered. {It is worth noting that in those epochs, the offered assortments in nests other than the $i$-th nest (i.e., the nests $i'$ for $i' \neq i$) can be arbitrary since the distributions of $\hat n_{i,\tau}$ and $\hat r_{i,\tau}$ are \emph{independent} of $\hat\theta_{i'}$ for $i'\neq i$.
    }
	{This key independence property enables us to combine purchasing information of vastly different assortment combinations (provided that $\hat\theta_i$ remains the same), which forms an important \emph{aggregation strategy} that avoids exponentially large regret if assortment combinations are treated independently.}
	
	
	
	\subsection{Efficient computation of $\hat{\vct\theta}$}\label{subsec:optimization}
	
	Our policy in Algorithm \ref{alg:ucb} involves a combinatorial optimization problem over all $\vct\theta\in\mathcal K_1\times\cdots\times\mathcal K_M$ (see Step \ref{line:alg-nested-improved-max} in Algorithm \ref{alg:ucb}).
	A brute-force algorithm that enumerates all such $\vct\theta$ takes $O(K^M)$ time and is computationally intractable even for small $M$ values,
	{where $K=\max_i|\{r_{ij}: j\in[N_i]\}|\leq N+1$.}
	In this section we introduce a computationally efficient procedure to compute $\hat{\vct\theta}$ by using a binary search technique. The idea behind our procedure is similar to the one introduced in \citep{Rusmevichientong2010} for dynamic assortment optimization in MNL models,
	which can also be traced to the fractional programming work as early as \citep{megiddo1978combinatorial}.
	
	For any $\lambda\in[0,1]$ and $\vct\theta=(\theta_1,\cdots,\theta_M)\in\mathcal K_1\times\cdots\times\mathcal K_M$ define potential function
	\begin{equation}
	\psi_\lambda(\vct\theta) := \sum_{i=1}^M(\bar\phi_{i,\theta_i}-\lambda)\bar u_{i,\theta_i}.
	\label{eq:potential}
	\end{equation}
	
Recall the definition of $\bar R'(\vct\theta) = \frac{\sum_{i=1}^M\bar\phi_{i,\theta_i}\bar u_{i,\theta_i}}{1+\sum_{i=1}^M\bar u_{i,\theta_i}}$ in	Step  \ref{line:alg-nested-improved-max} of Algorithm \ref{alg:ucb}. The following lemma characterizes the properties of $\psi_\lambda(\vct\theta) $ and its relationship with $\bar R^*=\max_{\vct\theta\in\mathcal K_1\times\cdots\times\mathcal K_M}\bar R'(\vct\theta)$:
	\begin{lemma}
	The following holds for all $\lambda\in[0,1]$:
	\begin{enumerate}
	\item If $\bar R^*\geq\lambda$, then there exists a $\vct\theta\in\mathcal K_1\times\cdots\times\mathcal K_M$ such that $\psi_\lambda(\vct\theta)\geq\lambda$;
	furthermore if $\bar R^*>\lambda$, then the inequality is strict;
	\item If $\bar R^*\leq\lambda$, then for all $\vct\theta\in\mathcal K_1\times\cdots\times\mathcal K_M$, $\psi_\lambda(\vct\theta)\leq\lambda$;
	furthermore if $\bar R^*<\lambda$, then the inequalities are strict.
	\end{enumerate}
	\label{lem:binary-search}
	\end{lemma}
	
	Based on Lemma \ref{lem:binary-search}, an efficient optimization algorithm computing the maximizer $\hat{\vct\theta}^{(\tau)}$ can be designed
	by a \emph{binary search} over $\lambda\in[0,1]$. In particular,
	for each fixed value of $\lambda$, the $\vct\theta^*(\lambda)=(\theta_1^*(\lambda),\cdots,\theta_M^*(\lambda))\in\mathcal K_1\times\cdots\times\mathcal K_M$ that maximizes $\psi_\lambda(\vct\theta)$
	can be found by setting $\theta_i^*(\lambda) \in \arg\max_{\theta_i\in\mathcal K_i}(\bar\phi_{i,\theta_i}-\lambda)\bar u_{i,\theta_i}$.
	If $\psi_\lambda(\vct\theta^*(\lambda)) > \lambda$, then $\bar R^* >\lambda$, because otherwise it violates the second property in Lemma \ref{lem:binary-search}.
	Similarly, if $\psi_\lambda(\vct\theta^*(\lambda))\leq \lambda$, then $\bar R^*\leq\lambda$, because otherwise it violates the second part of the first property in Lemma \ref{lem:binary-search} (note that since $\vct\theta^*(\lambda)$ is the maximizer of $\psi_\lambda(\vct\theta)$, $\psi_\lambda(\vct\theta^*(\lambda))\leq \lambda$ implies that $\psi_\lambda(\vct\theta) \leq \lambda$ for all $\vct\theta$).
	Thus, whether $\bar R^*>\lambda$ or $\bar R^*\leq\lambda$ can be determined by solely comparing $\psi_\lambda(\vct\theta^*(\lambda))$ with $\lambda$.
		
	We remark that each evaluation of $\psi_\lambda(\vct\theta^*(\lambda))$ takes $O(MK)$ time, and the entire binary search procedure
	takes time $O(MK\log(1/\epsilon))$ to approximate $\bar R^*$ up to arbitrarily small error $\epsilon$.
	This is much faster than the brute force algorithm that takes $O(K^M)$ time.
	
	{
	We also remark that, similar to other bisection type algorithms, the computation procedure outlined above computes \emph{approximate} solutions only,
	with $O(\log(1/\epsilon))$ iterations required if an error level of $\epsilon>0$ is desired.
	We suggest setting the accuracy level $\epsilon$ to
	$\epsilon=1/T$, which would inflate an additional $O(1)$ term in the cumulative regret upper bound,
	while the running time of the binary search routine is strictly polynomial in $T$.
	When the time horizon $T$ is unknown before hand,
	a doubling trick can be used to consider epochs of lengths $1,2,4,\cdots,2^\tau,\cdots,$
	and within epoch $\tau$ (of length $2^\tau$) an error level of $\epsilon_\tau = 2^{-\tau}$ can be used.
	}
	
	\subsection{Regret analysis}
	
	Below is our main regret theorem for Algorithm \ref{alg:ucb}. 
	\begin{theorem}
		{For each nest $i$ let $\mathcal K_i = \{r_{ij}: j\in[N_i]\}$.}
		The assortment sequence $\{\vct\theta^{(t)}\}_{t=1}^T$ produced by Algorithm \ref{alg:ucb} has the regret
		upper bounded as
		\begin{equation}
		\mathrm{Regret}(\{\vct\theta^{(t)}\}_{t=1}^T) \lesssim \sqrt{MKT\log(MKT)} + MKU\log^2(MKT) + O(1),
		\end{equation}
		where $K=\max_i|\mathcal K_i|$ and $U=\max_{i\in[M]}\max_{\theta\in\mathcal K_i}u_{i,\theta}$.
		\label{thm:ucb}
	\end{theorem}
	
	\begin{corollary}
	With $K=|\mathcal K_i|=N+1$ (for any $i\in [M]$) and $U\leq NC_V$, the regret upper bound in Theorem \ref{thm:ucb}
can be simplified to
	\begin{eqnarray}	
	\mathrm{Regret}(\{\mat S^{(t)}\}_{t=1}^T) & \lesssim &\sqrt{MNT\log(MNT)} + MN^2C_V\log^2(MNT) + O(1)
     \nonumber\\
    &= & \tilde{O}(\sqrt{MNT}+MN^2) \label{eq:regret-preliminary}
	\end{eqnarray}
	\label{cor:regret-preliminary}
	\end{corollary}
	
	We make several remarks on the regret upper bound in Corollary \ref{cor:regret-preliminary}. {In online and bandit learning literature, the time horizon $T$ is usually considered to be the dominating term asymptotically.}
	Therefore, when $T > M$ and the number of items per nest $N$ is small as compared to $T$, the dominating term in Eq.~(\ref{eq:regret-preliminary}) is $\tilde O(\sqrt{MNT})$.
	This matches the lower bound result $\Omega(\sqrt{MT})$ in Theorem \ref{thm:lower-bound} within a factor of $\sqrt{N}$.
	We give further discussion on this gap of $O(\sqrt{N})$ in Sec.~\ref{sec:lower}. We will also show later in Sec.~\ref{sec:discretization} how to deal with a large $N$ case by considering a ``discretization'' heuristic.
	%
	
	
	In the rest of the section we sketch key steps and lemmas toward the proof of Theorem \ref{thm:ucb}. The detailed proofs of these lemmas are provided in the supplementary material.
	First, the following lemma shows that the estimates
	$\hat \phi_{i,\theta},\hat u_{i,\theta}$ concentrate around the true values $\phi_{i,\theta},u_{i,\theta}$.
	\begin{lemma}
		Suppose $T(i,\theta)\geq 96\ln(2MTK)$.
		With probability $1-T^{-1}$ uniformly over all $i\in[M]$, $\theta\in\mathcal K_i$ and $t\in[T]$
		\begin{align}
		\big|\hat u_{i,\theta}-u_{i,\theta}\big| &\leq\min\left\{U, 3\sqrt{\frac{48\max(\hat u_{i,\theta},\hat u_{i,\theta}^2)\ln(2MTK)}{T(i,\theta)}} +\frac{144\ln(2MTK)}{T(i,\theta)} \right\};\label{eq:u_ci}\\
		\big|\hat\phi_{i,\theta}-\phi_{i,\theta}\big|&\leq\min\left\{1, \sqrt{\frac{\ln(2MTK)}{T(i,\theta)\hat u_{i,\theta}}}\right\}.\label{eq:phi_ci}
		\end{align}
		\label{lem:new-nested-concentration}
		In addition, if $u_{i,\theta}\geq 1$ then $\hat u_{i,\theta}\in[0.5u_{i,\theta},2u_{i,\theta}]$.
	\end{lemma}
	
	The following corollary is an immediate consequence of Lemma \ref{lem:new-nested-concentration}:
	\begin{corollary}
	{Suppose $T(i,\theta)\geq 96\ln(2MTK)$.}
		With probability $1-T^{-1}$, $\bar u_{i,\theta}\geq u_{i,\theta}$ and $\bar\phi_{i,\theta}\geq \phi_{i,\theta}$ for all $i\in[M]$, $\theta\in\mathcal K_1\times\cdots\times\mathcal K_M$.
		\label{cor:new-nested-concentration}
	\end{corollary}

	Corollary \ref{cor:new-nested-concentration} shows that (with high probability) $\bar u_{i,\theta}$ and $\bar\phi_{i,\theta}$ are valid
	upper bounds for $u_{i,\theta}$ and $\phi_{i,\theta}$.
	Our next corollary shows that $\bar R'$ is also an upper bound for $R'$ at maximizers of $\bar R'$ and $\bar R$.
	Recall that $\bar R'(\vct\theta)=[\sum_{i=1}^M\bar\phi_{i,\theta_i}\bar u_{i,\theta_i}]/[1+\sum_{i=1}^M\bar u_{i,\theta_i}]$ and
	$R'(\vct\theta) = [\sum_{i=1}^M\phi_{i,\theta_i}u_{i,\theta_i}]/[1+\sum_{i=1}^Mu_{i,\theta_i}]$.
	
	We defer its proof to the online supplement.
	\begin{corollary}
	With probability $1-T^{-1}$, $\bar R'(\hat{\vct\theta})\geq R'(\hat{\vct\theta})$ and $\bar R'(\vct\theta^*)\geq R'(\vct\theta^*)$,
	where $\hat{\vct\theta},\vct\theta^*\in\mathcal K_1\times\cdots\times\mathcal K_M$ are maximizers of $\bar R'$ and $R'$, respectively.
	\label{cor:Rub}
	\end{corollary}
	
	We are now ready to sketch the proof of Theorem \ref{thm:ucb}.
	The first step is to use the classical regret decomposition for UCB-type policies ($\mathcal A$ denotes the success event in Corollary \ref{cor:Rub}).
	\begin{align}
	\mathrm{Regret}(\{\hat{\vct\theta}^{(t)}\}_{t=1}^T)
	&= \mathbb E\sum_{t=1}^T R'(\vct\theta^*) - R'({\vct\theta}^{(t)})\nonumber\\
	&\leq \mathbb E\left[\sum_{t=1}^T R'(\vct\theta^*)-R'({\vct\theta}^{(t)})\bigg|\mathcal A\right]\Pr[\mathcal A] + O(T)\cdot \Pr[\mathcal A^c]\nonumber\\
	&\leq O(1) + \mathbb E\left[\sum_{t=1}^T \bar R'(\vct\theta^*)-\bar R'({\vct\theta}^{(t)}) + \bar R'({\vct\theta}^{(t)})-R'({\vct\theta}^{(t)})\bigg|\mathcal A\right]\label{eq:regret-intermediate-1}\\
	&\leq O(1) + \mathbb E\left[\sum_{t=1}^T \bar R'({\vct\theta}^{(t)})-R'({\vct\theta}^{(t)})\bigg|\mathcal A\right].\label{eq:regret-intermediate-2}\\
	&= O(1) + \mathbb E\left[\sum_\tau |\mathcal E_\tau|\cdot (\bar R'(\hat{\vct\theta}^{(\tau)})-R'(\hat{\vct\theta}^{(\tau)})) \bigg|\mathcal A \right].\label{eq:regret-intermediate-25}
	\end{align}
    Here, $\hat{\vct\theta}^{(\tau)}$ denotes any ${\vct\theta}^{(t)}$ in the $\tau$-th epoch $\mathcal E_\tau$ \footnote{Recall that in Algorithm \ref{alg:ucb}, ${\vct\theta}^{(t)}$ does not change within the same epoch $\mathcal E_\tau$. We write $\hat{\vct\theta}^{(\tau)}$ to highlight that $\hat{\vct\theta}^{(\tau)}$ is the maximizer of $\bar{R}'$ in the $\tau$-th epoch (see Step \ref{line:alg-nested-improved-max} of Algorithm \ref{alg:ucb}).}.
	We also note that  Eq.~(\ref{eq:regret-intermediate-1}) holds because $\Pr[\mathcal A^c]\leq T^{-1}$ and $\bar R'(\vct\theta^*)\geq R'(\vct\theta^*)$, and
	Eq.~(\ref{eq:regret-intermediate-2}) holds because $\bar R'(\vct\theta^{(t)}) \geq \bar R'(\vct\theta^*)$, since $\vct\theta^{(t)}$ is the maximizer of $\bar R'$ at time $t$.
	
	It remains to upper bound the discrepancy between $\bar R'(\hat{\vct\theta}^{(\tau)})$ and $R'(\hat{\vct\theta}^{(\tau)})$ at every epoch $\tau$.
	This is accomplished by the following ``aggregation lemma'', which is proved in the online supplement.
	\begin{lemma}
	With probability $1-T^{-1}$, for all $t\in[T]$, $i\in[M]$ and $\vct\theta=(\theta_1,\cdots,\theta_M)\in\mathcal K_1\times\cdots\times\mathcal K_M$,
	\begin{equation}
	\bar R'(\vct\theta)-R'(\vct\theta)\leq \frac{1}{1+\sum_{i=1}^Mu_{i,\theta_i}}\left[ \sum_{i=1}^M\frac{\bar u_{i,\theta_i}-u_{i,\theta_i}}{1+u_{i,\theta_i}} + \sum_{i=1}^M u_{i,\theta_i}(\bar\phi_{i,\theta_i}-\phi_{i,\theta_i})\right].
	\label{eq:aggregation}
	\end{equation}
	\label{lem:aggregation}
	\end{lemma}
	{
	\begin{remark}
	Comparing Lemma \ref{lem:aggregation} with Lemma A.4 from \citep{Agrawal16MNLBandit},
we can see that there is an additional $1/[1+\sum_{i\in\mathcal M}u_{i,\theta}]$ multiplication term in the error upper bounds.
Such an improvement is made possible by our more careful analysis and insights into the mathematical structures of the MNL choice model,
and is important in dealing with preference parameters $v_{ij}$ larger than one.
	\end{remark}
	}

	Note that $\mathbb E|\mathcal E_\tau|=1 + \sum_{i=1}^M\mathbb E[\hat n_{i,\tau}] = 1 + \sum_{i=1}^Mu_{i,\theta_i}$.
	Combining Lemma \ref{lem:aggregation} with Eq.~(\ref{eq:regret-intermediate-25}) we obtain
	\begin{equation}
	\mathrm{Regret}(\{\hat{\vct\theta}^{(t)}\}_{t=1}^T)
	\leq O(1) + \sum_{\tau} \mathbb{E} \left[ \sum_{i=1}^M\frac{\bar u_{i,\hat \theta_i^{(\tau)}}-u_{i,\hat \theta_i^{(\tau)}}}{1+u_{i,\hat \theta_i^{(\tau)}}} + \sum_{i=1}^M u_{i,\hat \theta_i^{(\tau)}}(\bar\phi_{i,\hat \theta_i^{(\tau)}}-\phi_{i,\hat \theta_i^{(\tau)}})\bigg| \mathcal A\right].
	\label{eq:regret-intermediate-26}
	\end{equation}
	
	The following lemmas upper bound (asymptotically) the two terms in Eq.~(\ref{eq:regret-intermediate-26}) separately.
	\begin{lemma}
	Conditioned on event $\mathcal A$, it holds that
	\begin{equation}
	\sum_{\tau}\sum_{i=1}^M\frac{\bar u_{i,\hat \theta_i^{(\tau)}}-u_{i,\hat \theta_i^{(\tau)}}}{1+u_{i,\hat \theta_i^{(\tau)}}}
	\lesssim \sqrt{MKT\log(MTK)} + MKU\log^2(MTK).
	\label{eq:regret-final-part1}
	\end{equation}
	\label{lem:regret-final-part1}
	\end{lemma}
	
	\begin{lemma}
	Conditioned on event $\mathcal A$, it holds that
	\begin{equation}
	\sum_{\tau}\sum_{i=1}^Mu_{i,\hat \theta_i^{(\tau)}}(\bar\phi_{i,\hat \theta_i^{(\tau)}}-\phi_{i,\hat \theta_i^{(\tau)}})
	\lesssim \sqrt{MKT\log(MTK)} + MKU\log^2(MTK).
	\label{eq:regret-final-part2}
	\end{equation}
	\label{lem:regret-final-part2}
	\end{lemma}
	
	Lemmas \ref{lem:regret-final-part1} and \ref{lem:regret-final-part2} are proved in the supplementary material.
	Combining both lemmas and Eq.~(\ref{eq:regret-intermediate-26}), we complete the proof of Theorem \ref{thm:ucb}.

\subsection{A discretization heuristic}
\label{sec:discretization}
When the number of items $N$ per nest is large,  we present a useful discretization heuristic that \emph{discretizes} the parameter sets $\mathcal K_i$ into small finite subsets. In other words,  instead of considering level sets defined for thresholds $\theta=r_{ij}$ for all $j\in[N]$ so that $|\mathcal K_i|=N+1$,
we only include level sets whose thresholds are on a finite grid.
Our simulation studies (see Sec.~\ref{sec:numerical}) demonstrate the effectiveness of this method.


More specifically, let $\delta\in(0,1)$ be a granularity parameter to be specified by the retailer. 
Recall the definition of the level set $\mathcal L_i(\theta)=\{j\in[N]: r_{ij}\geq \theta\}$.
In the discretized framework,
we only consider level set threshold parameters $\theta$ that are multiples of $1/\delta$. Let $\mathbb N$ be the set of non-negative integers and define
{\begin{equation}
\tilde{\mathcal K}_i^\delta := \left\{\theta: 0\leq\theta\leq 1,\;\;\theta/\delta\in\mathbb N,\;\; \mathcal L_i(\theta)'s \;\; \text{are distinct} \right\} \cup \{\infty\}, \quad \text{for} \;\; i \in [M]
\label{eq:discretization}
\end{equation}
where each $\theta \in \tilde{\mathcal K}_i^\delta$ corresponds to a \emph{unique} level set $\mathcal{L}_i(\theta)$.
When there are multiple $\theta$'s leading to the same level set, we keep any one of these $\theta$'s in $\tilde{\mathcal K}_i^\delta$ and thus the level sets induced by $\tilde{\mathcal K}_i^\delta$ (i.e., $\{\mathcal{L}_i(\theta): \theta \in \tilde{\mathcal K}_i^\delta\}$) are unique.}
Since duplicate assortment sets are removed in $\tilde{\mathcal K}_i^\delta$,  we have $\tilde{\mathcal K}_i^\delta\subseteq\mathcal K_i$  and thus $|\tilde{\mathcal K}_i^\delta|\leq |\mathcal K_i|=K=N+1$.
Moreover, we also have $|\tilde{\mathcal K}_i^\delta|\leq \lfloor1/\delta\rfloor + 2$ because level set thresholds in $\tilde{\mathcal K}_i^\delta$
must be an integer multiple of $\delta$. On one hand, when $\delta$ is not too small, the size of $\tilde{\mathcal K}_i^\delta$ could be significantly smaller than $N$. On the other hand,  when $\delta\to 0$, we recover the original set $\mathcal K_i$, which gives the full level sets. We shall thus define $\tilde{\mathcal K}_i^\delta := \mathcal K_i$ when $\delta=0$.



The following discretized reduction lemma shows that
by restricting ourselves to $\tilde{\mathcal K}_i^\delta$ instead of $\mathcal K_i$,
the approximation error in terms of  expected revenue can be upper bounded by $\delta$, which goes to zero as we take $\delta\to 0$.
\begin{lemma}[Discretized reduction lemma]
	Fix an arbitrary $\delta\in(0,1)$. Then 
	$$
	\max_{\vct\theta\in\mathcal K_1\times\cdots\times\mathcal K_M} R'(\vct\theta) - \max_{\vct\theta\in\tilde{\mathcal K}_1^\delta\times\cdots\times\tilde{\mathcal K}_M^\delta} R'(\vct\theta) \leq\delta,
	$$
	where $R'(\vct\theta) := [\sum_{i=1}^M\phi_{i,\theta_i}u_{i,\theta_i}]/[1+\sum_{i=1}^Mu_{i,\theta_i}]$.
	\label{lem:discretization}
\end{lemma}

{With a pre-specified $\delta$, we run the policy in Algorithm \ref{alg:ucb} on the parameter space $\tilde{\mathcal K}_1^\delta\times\cdots\times\tilde{\mathcal K}_M^\delta$.}
As a result of Lemma \ref{lem:discretization}, the value of $\delta$ can be thought of as a tradeoff between additive bias and multiplicative terms in the final regret.
With a small value of $\delta$, there is almost no additive terms arising from Lemma \ref{lem:discretization}, yet the number of items $N$ per nest will not be reduced too much.
On the other hand, when $\delta$ is large the regret bound in Corollary \ref{cor:regret-preliminary} is improved as the number of items $N$ per nest is now upper bounded by $\lfloor1/\delta\rfloor + 2$.
However, a large $\delta$ value will introduce a large additive bias from Lemma \ref{lem:discretization}.
Hence, a balance has to be achieved for an appropriate value of $\delta$ to deliver the best performance.  {We further demonstrate the performance for different choices of $\delta$ in our simulation studies (see Sec.~\ref{sec:numerical}).}

\section{A regret lower bound}
\label{sec:lower}
We establish the following lower bound on the regret of any dynamic assortment planning policy under nested logit models.
\begin{theorem}
Suppose the number of nests $M$ is divisible by 4 and $\gamma_1=\cdots=\gamma_M=0.5$. {Assume also that (A1) and (A2) hold.}
Then there exists a numerical constant $C_0>0$ such that for any dynamic assortment planning policy $\pi$,
\begin{equation}
\sup_{\{r_{ij},v_{ij}\}} \sum_{t=1}^T R^* - \mathbb E^\pi\left[R(\mat S^{(t)})\right] \geq C_0\sqrt{MT} \;\;\;\;\;\;\text{where}\;\; R^* = \max_{\mat S\in\mathbb S}R(\mat S).
\label{eq:lower-bound}
\end{equation}
\label{thm:lower-bound}
\end{theorem}
\begin{remark}
The condition that $M$ is divisible by 4 is only a technical condition and does not affect the main message delivered in Theorem \ref{thm:lower-bound}, which shows necessary dependency on $M$ asymptotically when $M$ is large.
\end{remark}

{
\begin{remark}{(Discussion on the dependency of $M$)}\;
Comparing Theorem \ref{thm:lower-bound} with the regret upper bound in Corollary \ref{cor:regret-preliminary},
we notice that when $T$ (time horizon) is large compared to $M$ (the number of nests), both regret bounds have
an $O(\sqrt{M})$ dependency on $M$.
This suggests that our algorithm and regret analysis delivers \emph{optimal} dependency of regret on the number of nests $M$
in a dynamic nested assortment planning problem.
\end{remark}

\begin{remark}{(Discussion on the dependency of $N$)} \;
Comparing Theorem \ref{thm:lower-bound} with the regret upper bound in Corollary \ref{cor:regret-preliminary}, we notice that there is a gap of $\sqrt{N}$ between the upper and lower bounds.

We conjecture that the \emph{upper bound} with an additional $O(\sqrt{N})$ factor is in fact tight. Actually, because our proposed algorithm
treats each ``level set'' assortments (within each nest) as standalone estimation units, it is intuitive to see that the regret that \emph{our algorithm incurs}
has to scale polynomially with $N$. We conjecture that \emph{any} possible dynamic strategy for nested logit models
has to suffer at least an $O(\sqrt{N})$ term in regret bound.

Unfortunately, due to technical difficulty of constructing lower bounds for problem instances, 
we are unable to extend our lower bound constructions to more than $N=3$ items per nest.
This is because our lower bound construction (to be presented later) uses only $N=3$ items per nest and therefore cannot deliver a lower bound depending on $N$.
We thus leave the question of proving a matching $O(\sqrt{MNT})$ lower bound as an interesting yet challenging open problem.
\label{rem:lowerbound}
\end{remark}
}

In the rest of this section, we provide the proof of Theorem \ref{thm:lower-bound}, while deferring proofs of several technical lemmas to the supplementary material.

\subsection{Construction of adversarial model parameters}\label{subsec:lower-bound-construction}

\begin{table}[t]
\centering
\caption{Adversarial construction of two types of nests.
The revenue parameter $\rho$ is set to $\rho=9\sqrt{2}/(1+\sqrt{2}) \approx 0.694774$.}
\begin{tabular}{lccccccc}
\hline
& \multicolumn{3}{c}{Type A Nest}& & \multicolumn{3}{c}{Type B Nest}\\
\cline{2-4}\cline{6-8}
& Item 1&Item 2& Item 3& & Item 1& Item 2& Item 3\\
\hline
Revenues $r_{ij}$&1& 0.8& $\rho$& & 1& 0.8& $\rho$ \\
Preferences $v_{ij}$& $(1+\epsilon)/M^2$& $(1-\epsilon)/M^2$& $1/M^2$& & $(1-\epsilon)/M^2$&  $(1+\epsilon)/M^2$ & $1/M^2$ \\
\hline
\end{tabular}
\label{tab:construction}
\end{table}

Let $\epsilon>0$ be a small positive parameter depending on $M$ and $T$, which will be specified later.
Each nest $i\in[M]$ in our construction consists of $N=3$ items and is classified into two categories: ``Type A'' and ``Type B'',
with parameter configurations detailed in Table \ref{tab:construction}.
Note that regardless of which type of nest $i\in[M]$ is, the three items in nest $i$ have preference parameters $(1+\epsilon)/M^2$, $(1-\epsilon)/M^2$ and $1/M^2$.
Hence it is impossible to decide the type of a nest without observations of customers' purchasing actions. {Given the model parameters in Table \ref{tab:construction}, it is easy to verify that for a Type A nest, the optimal assortment is $\{1,2\}$, while for a Type B nest, the optimal assortment is $\{1,2,3\}$.}

The following lemma shows that any assortment $S_i$ that does not equal $\{1,2\}$ for Type A nests or $\{1,2,3\}$ for Type B nests
incurs an $\Omega(\epsilon/M)$ regret.
It is proved in the supplementary material.
\begin{lemma}
Let $U\subseteq[M]$ be the set of Type A nests, and by construction $[M]\backslash U$ are all Type B nests.
For any $\mat S=(S_1,\cdots,S_M)\in [N]^M$, define $m_U^\sharp(\mat S) := \sum_{i\in U}\vct 1\{S_i\neq \{1,2\}\} + \sum_{i\notin U}\vct 1\{S_i\neq \{1,2,3\}\}$.
Then there exists a numerical constant $C>0$ such that for all $\mat S$, $R(\mat S^*)-R(\mat S) \geq m_U^\sharp(\mat S)\cdot C\epsilon/M$,
where $\mat S^*\in\arg\max_{\mat S}R(\mat S)$ is the optimal assortment combination under $U$.
\label{lem:S-diff}
\end{lemma}

To avoid confusion, we emphasize that in our \emph{lower bound proof} the notation $U$ refers to a particular type of nest,
instead of upper confidence bounds in algorithm descriptions and the upper bound proof.

\subsection{Reduction to average-case regret}
{

For any policy $\pi$, we want to show a lower bound on the \emph{worst-case} regret
\begin{equation}
\sup_{\{r_{ij},v_{ij}\}} \sum_{t=1}^TR^* - \mathbb E^\pi\left[R(\mat S^{(t)})\right].
\label{eq:worst-case}
\end{equation}

Recall that in our adversarial construction, $U\subseteq[M]$ denotes the set of all Type A nests and the remaining nests $[M]\backslash U$ are  Type B.
The following inequalities show a reduction to average-case regret:
\begin{equation}
\sup_{\{r_{ij},v_{ij}\}} \sum_{t=1}^TR^* - \mathbb E^\pi\left[R(\mat S^{(t)})\right]
\geq \sup_{U\subseteq[M]} \sum_{t=1}^TR^* - \mathbb E_U^\pi\left[R(\mat S^{(t)})\right]
\geq \frac{1}{2^M}\sum_{U\subseteq[M]} \sum_{t=1}^TR^* - \mathbb E^\pi_U\left[R(\mat S^{(t)})\right],
\label{eq:reduction-avg}
\end{equation}
where in $\sup_{U\subseteq[M]}$ or $\sum_{U\subseteq[M]}$ we are optimizing or summing over all $2^M$ subsets of $[M]=\{1,2,\cdots,M\}$.
Here we also use the $\mathbb E_U^\pi$ notation to emphasize that the distribution of $\{\mat S^{(t)}\}$ (and hence the expectation)
depends on both the parameter setting (uniquely determined by the set of Type A nests $U\subseteq[M]$) and the policy $\pi$ itself.

For any $i\in[M]$ and $S\subseteq[N]$, denote $\mn_S(i) := \sum_{t=1}^T\vct 1\{S_i^{(t)}=S\}$ as the random variable
of the number of times assortment $S$ is offered in nest $i$.
Let $\mathbb E_U^\pi[\mn_S(i)]$ be the expectation of $\mn_S(i)$, with expectation taken under model parameters setting $U$ (recall that $U$ is the set of all Type A nests)
and policy $\pi$.
Invoking Lemma \ref{lem:S-diff} and noting that $\sum_{S\subseteq[N]}\mathbb E_U^\pi[\mn_S(i)]= T$ for any $U\subseteq[M]$, $i\in[M]$ and policy $\pi$,
the right-hand side of Eq.~(\ref{eq:reduction-avg}) can be lower bounded by,
\begin{align}
\frac{1}{2^M}\sum_{U\subseteq[M]}&\sum_{t=1}^T\mathbb E^\pi\left[m_U^\sharp(\mat S^{(t)})\cdot \frac{C\epsilon}{M}\right]\nonumber\\
&= \frac{C\epsilon}{M}\frac{1}{2^M}\sum_{U\subseteq[M]}\left[\sum_{i\in U}\sum_{S\neq\{1,2\}}\mathbb E_U^\pi[\mn_S(i)] + \sum_{i\notin U}\sum_{S\neq\{1,2,3\}}\mathbb E_U^\pi[\mn_S(i)]\right]\nonumber\\
&\geq \frac{C\epsilon}{M}\frac{1}{2^M}\sum_{U\subseteq[M]}\left[\sum_{i\in U}\sum_{S\neq\{1,2\}}\mathbb E_U^\pi[\mn_S(i)] + \sum_{i\notin U}\mathbb E_U^\pi[\mn_{\{1,2\}}(i)]\right]\nonumber\\
&= \frac{C\epsilon}{M}\frac{1}{2^M}\sum_{U\subseteq[M]}\left[\sum_{i\in U}(T-\mathbb E_U^\pi[\mn_{\{1,2\}}(i)])+ \sum_{i\notin U}\mathbb E_U^\pi[\mn_{\{1,2\}}(i)]\right]\label{eq:proof-reduction-1}\\
&= \frac{C\epsilon}{M}\frac{1}{2^M}\left(2^M\times \frac{MT}{2}-\sum_{U\subseteq[M]}\left[\sum_{i\in U}\mathbb E_U^\pi[\mn_{\{1,2\}}(i)]- \sum_{i\notin U}\mathbb E_U^\pi[\mn_{\{1,2\}}(i)]\right]\right)\label{eq:proof-reduction-2}\\
&= \frac{C\epsilon T}{2} - \frac{C\epsilon}{M}\frac{1}{2^M}\sum_{U\subseteq[M]}\sum_{i=1}^M (-1)^{\vct 1\{i\in U\}}\times\mathbb E_U^\pi[\mn_{\{1,2\}}(i)].
\label{eq:reduction-avg-intermediate}
\end{align}
Here in Eq.~(\ref{eq:proof-reduction-1}) we use the fact that $\sum_{S\subseteq[N]}\mathbb E_u^\pi[\mn_S(i)]=T$;
Eq.~(\ref{eq:proof-reduction-2}) holds because $\sum_{U\subseteq[M]}\sum_{i\in U}T = \sum_{U\subseteq[M]}\sum_{i\notin U}T$ by symmetry,
and furthermore $\sum_{U\subseteq[M]}(\sum_{i\in U}T+\sum_{i\notin U}T) = \sum_{U\subseteq[M]}\sum_{i=1}^MT = 2^M\times MT$.


Next, for every $U\subseteq[M]$, define $U'=U\oplus i$ as $U'=U\cup\{i\}$ if $i\notin U$, and $U'= U\backslash\{i\}$ if $i\in U$.
Clearly, there is a one-to-one correspondence between $U\subseteq[M]$ and $U\oplus i\subseteq[M]$, for every fixed $i\in[M]$.
The right-hand side of Eq.~(\ref{eq:reduction-avg-intermediate}) can then be simplified as
\begin{align}
&\frac{C\epsilon}{2} - \frac{C\epsilon}{M}\frac{1}{2^M}\sum_{i=1}^M \frac{1}{2}\left[\sum_{U\subseteq[M]}(-1)^{\vct 1\{i\in U\}}\mathbb E_U^\pi[\mn\{1,2\}(i)]
+ \sum_{U\subseteq[M]}(-1)^{\vct 1\{i\in U\oplus i\}}\mathbb E_{U\oplus i}^\pi[\mn\{1,2\}(i)]\right]\nonumber\\
&= \frac{C\epsilon}{2} - \frac{C\epsilon}{M}\frac{1}{2^{M+1}}\sum_{i=1}^M\sum_{U\subseteq[M]}(-1)^{\vct 1\{i\in U\}}\times \big(\mathbb E_U^\pi[\mn_{\{1,2\}}(i)] - \mathbb E_{U\oplus i}^\pi[\mn_{\{1,2\}}(i)]\big).
\label{eq:reduction-avg-final}
\end{align}
}

\subsection{Pinsker's inequality}\label{subsec:pinsker}

Let $P_U^\pi,P_W^\pi$ denote the probabilistic laws under $U$, $W$ and policy $\pi$. Then
for any $S\subseteq[N]$,
\begin{align}
\big|\mathbb E_{U}^\pi[\mn_S(i)]&-\mathbb E_{W}^\pi[\mn_S(i)]\big|
\leq \sum_{j=0}^T j\cdot\big|P_U^\pi[\mn_S(i)=j] - P_{W}^\pi[\mn_S(i)=j]\big| \leq T\cdot  \sum_{j=0}^T\big|P_U^\pi[\mn_S(i)=j] - P_{W}^\pi[\mn_S(i)=j]\big|\nonumber\\
&= T\|P_U^\pi-P_{W}^\pi\|_{\mathrm{TV}} \leq T\sqrt{\frac{1}{2}\min\{\kl(P_U^\pi\|P_{W}^\pi),\kl(P_W^\pi\|P_U^\pi)\}} \label{eq:pinsker}\\
&\leq T\sqrt{\frac{T}{2}\min\{\max_{\mat S}\kl(P_U(\cdot|\mat S)\|P_{W}(\cdot|\mat S)), \max_{\mat S}\kl(P_W(\cdot|\mat S)\|P_U(\cdot|\mat S))\}}.
\label{eq:lb-kl}
\end{align}
{Here $\|P-Q\|_{\mathrm{TV}}$ and $\kl(P||Q)$ denote the total variational distance and Kullback-Leibler divergence between two probability laws $P$ and $Q$.
Eq. \eqref{eq:pinsker} is known as the \emph{Pinsker's inequality} (see e.g., \cite{tsybakov2009introduction,csiszar2011information}).}
Note that in the last term $P_U$ and $P_{W}$ do not have superscript $\pi$, because conditioned on a particular assortment combination $\mat S$
the KL divergence no longer depends on $\pi$.

The following lemma shows that if $U$ and $W$ differ by only one nest, then the KL divergence between $P_U$ and $P_{W}$ is small
\emph{for all $\mat S=(S_1,\cdots,S_M)$}.
\begin{lemma}\label{lemma:KL}
Suppose $|U\triangle W|=1$, where $U\triangle W = (U\backslash W)\cup(W\backslash U)$ is the symmetric difference between subsets $U,W\subseteq[M]$.
Then there exists a universal constant $C'>0$ such that for any $\mat S=(S_1,\cdots,S_M)$, $\min\{\kl(P_U(\cdot|\mat S)\|P_{W}(\cdot|\mat S),\kl(P_W(\cdot|\mat S)\|P_U(\cdot|\mat S))\}) \leq C'\epsilon^2/M$.
\end{lemma}

Invoking Lemma \ref{lemma:KL}, the right-hand side of Eq.~(\ref{eq:lb-kl}) can be further upper bounded by
\begin{align}
T\sqrt{\frac{T}{2}\cdot \frac{C'\epsilon^2}{M}} \lesssim T\sqrt{T\epsilon^2/M}.
\label{eq:lb-kl-simplified}
\end{align}

{
We are now ready to prove Theorem \ref{thm:lower-bound} by simplifying the Eq.~(\ref{eq:reduction-avg-final}) with the help of 
Eqs.~(\ref{eq:lb-kl}) and (\ref{eq:lb-kl-simplified}).
For every $U\subseteq[M]$ and $i\in[M]$, by Eq.~(\ref{eq:lb-kl}) it holds that
$$
\big|\mathbb E_U^\pi[\mn_{\{1,2\}}(i)] - \mathbb E_{U\oplus i}^\pi[\mn_{\{1,2\}}(i)] \lesssim T\sqrt{T\epsilon^2/ M}.
$$
Subsequently, Eq.~(\ref{eq:reduction-avg-final}) can be lower bounded as
$$
\frac{C\epsilon}{2} - \frac{C\epsilon}{M2^{M+1}}\sum_{i=1}^M\sum_{U\subseteq[M]} O(T\sqrt{T\epsilon^2 M})
\geq \frac{C\epsilon}{2} - C\epsilon\times O(T\sqrt{T\epsilon^2/M}).
$$
Setting $\epsilon = c_0\sqrt{M/T}$ for some sufficiently small positive constant $c_0>0$, 
the above inequality is lower bounded by $\Omega(\epsilon T) = \Omega(\sqrt{M/T})$.
This completes the proof of Theorem \ref{thm:lower-bound}.
}





\newcolumntype{G}{>{\centering}p{0.04\textwidth}}
	
\begin{table}[!t]

\scriptsize
\centering
\caption{Median (\textsc{Med}) and Maximum (\textsc{Max}) accumulated regret (summation over $T$ periods) for various algorithms and under various model and parameter settings. The minimum regret for each case is highlighted using the bold font. \textsf{TS} stands for Thompson Sampling, and \textsf{Exp-Exp} stands for Explore-then-Exploit.}
\begin{tabular}{lGGcGGcGGcGGcGGcGGcGc}
\hline
& \multicolumn{2}{c}{$\delta=0$}& &\multicolumn{2}{c}{$\delta=10^{-3}$}& & \multicolumn{2}{c}{$\delta=5\times 10^{-3}$}& & \multicolumn{2}{c}{$\delta= 10^{-2}$}& & \multicolumn{2}{c}{$\delta=5\times 10^{-2}$}& & \multicolumn{2}{c}{\textsf{TS}}& & \multicolumn{2}{c}{\textsf{Exp-Exp}}\\
\cline{2-3}\cline{5-6}\cline{8-9}\cline{11-12}\cline{14-15}\cline{17-18}\cline{20-21}
$(M,N)$& \textsc{Med}& \textsc{Max}& & \textsc{Med}& \textsc{Max}& & \textsc{Med}& \textsc{Max}& & \textsc{Med}& \textsc{Max}& & \textsc{Med}& \textsc{Max}& & \textsc{Med}& \textsc{Max}& & \textsc{Med}& \textsc{Max}\\
\hline
\multicolumn{2}{l}{\emph{$T=100$:}}\\
(5,100) & 5.5 & 6.4 & & 5.5 & 6.0 & & 3.8 & \bf 4.1 & & \bf 3.2 & 4.3 & & 5.4 & 8.5 & & 6.4 & 10.2& & 6.4 & 18.8\\
(10,100) & 4.8 & 6.2 & & 5.4 & 5.5 & & 4.7 & 6.5 & & \bf 2.3 & \bf 3.9 & & 5.8 & 7.0 & & 6.7 & 11.7& & 6.6 & 25.3\\
(5,250) & 10.4 & 14.1 & & 9.8 & 12.0 & & 5.7 & 6.5 & & \bf 3.3 & \bf 3.4 & & 7.0 & 8.3 & & 6.1 & 12.1& & 6.6 & 22.3 \\
(10,250) & 10.8 & 12.0 & & 9.7 & 12.3 & & 5.5 & 7.4 & &\bf 3.0 & \bf 4.4 & & 5.1 & 8.7 & & 6.8 & 9.2& & 6.1 & 17.4\\
(5,1000) & 22.0 & 25.3 & & 16.0 & 18.2 & & 6.2 & 7.5 & &\bf  3.2 & \bf 5.0 & & 6.9 & 10.9 & & 6.1 & 10.1& & 5.5 & 21.6 \\
(10,1000) & 21.5 & 24.1 & & 15.1 & 17.7 & & 5.1 & 6.4 & & \bf 3.1 & \bf 4.9 & & 6.2 & 9.4 & & 6.3 & 9.1 & & 6.4 & 24.8\\
\\
\hline
\multicolumn{2}{l}{\emph{$T=500$:}}\\
(5,100) & \bf 14.3 & \bf 18.5 & & 18.3 & 22.6 & & 26.8 & 30.9 & & 31.9 & 35.3 & & 33.3 & 34.3 & & 32.6 & 42.7& & 25.6 & 88.9 \\
(10,100) &\bf  15.7 & 23.0 & & 16.5 &\bf 22.1 & & 28.4 & 28.9 & & 35.4 & 36.5 & & 35.0 & 36.5 & & 33.0 & 42.9& & 30.7 & 128.0\\
(5,250) & 14.2 & 17.3 & &\bf  12.7 & \bf 14.9 & & 16.4 & 18.4 & & 29.1 & 36.8 & & 32.6 & 34.2 & & 30.6 & 43.4 & & 26.2 & 105.1\\
(10,250) & 13.8 & \bf 15.9 & & \bf 13.0 & 17.4 & & 16.6 & 19.6 & & 29.2 & 35.0 & & 35.8 & 38.6 & & 33.2 & 39.5 & & 30.5 & 47.4\\
(5,1000) & 41.1 & 46.1 & & 22.7 & 25.7 & & \bf 14.1 & \bf17.3 & & 29.4 & 37.4 & & 33.0 & 35.8 & & 30.8 & 39.8& & 27.0 & 94.3\\
(10,1000) & 39.3 & 44.2 & & 21.0 & 27.2 & & \bf 13.7 & \bf18.7 & & 28.0 & 37.0 & & 35.7 & 41.5 & & 32.4 & 39.8\\
& &29.3 & 49.4\\
\hline
\multicolumn{2}{l}{\emph{$T=10000$:}}\\
(5,100) & 491.5 & 505.5 & & \bf 489.4 & \bf 496.5 & & 494.5 & 500.8 & & 503.1 & 511.8 & & 513.4 & 525.2  & & 579.9 & 672.2& & 538.9 & 904.7\\
(10,100) & 548.4 & 558.0 & & 548.6 & 552.9 & & \bf 529.3 &\bf 534.7 & & 538.2 & 544.3 & & 554.3 & 565.2 & & 618.0 & 706.7& & 572.4 & 883.1\\
(5,250) & 534.4 & 543.7 & & 529.7 & 543.9 & & 523.4 & 536.1 & & 519.7 & \bf 525.5 & & 526.1 & 532.2 & & 617.7 & 694.3& & \bf 477.9 & 970.9\\
(10,250) & 551.0 & 560.5 & & 554.5 & 563.3 & & \bf 547.4 & 555.2 & & 548.6 & \bf 555.1 & & 571.6 & 578.4 & & 642.1 & 714.4& & 558.6 & 928.6\\
(5,1000) & 669.0 & 704.4 & & 570.5 & 584.8 & & 538.8 & 552.7 & &  532.9 &\bf  541.3 & & 535.8 & 558.4 & & 621.6 & 671.8& & \bf 489.7 & 863.3\\
(10,1000) & 703.5 & 738.2 & & 613.1 & 633.6 & & 555.7 & 566.3 & & \bf 549.9 & \bf 559.5 & & 567.2 & 578.6 & & 646.6 & 697.4 & &560.7 & 911.0\\
\hline
\end{tabular}
\label{tab:main}
\end{table}

\begin{figure}[!t]
\centering
\includegraphics[width=0.48\linewidth]{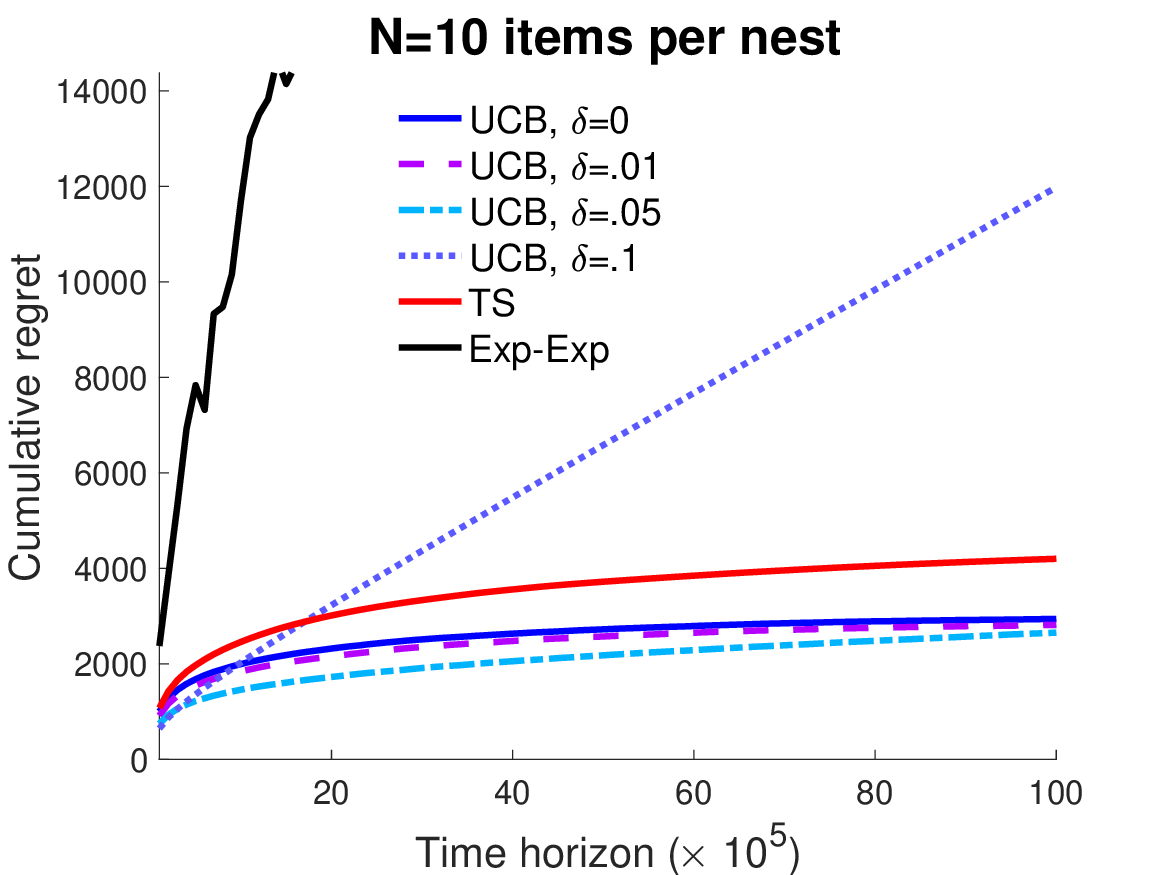}
\includegraphics[width=0.48\linewidth]{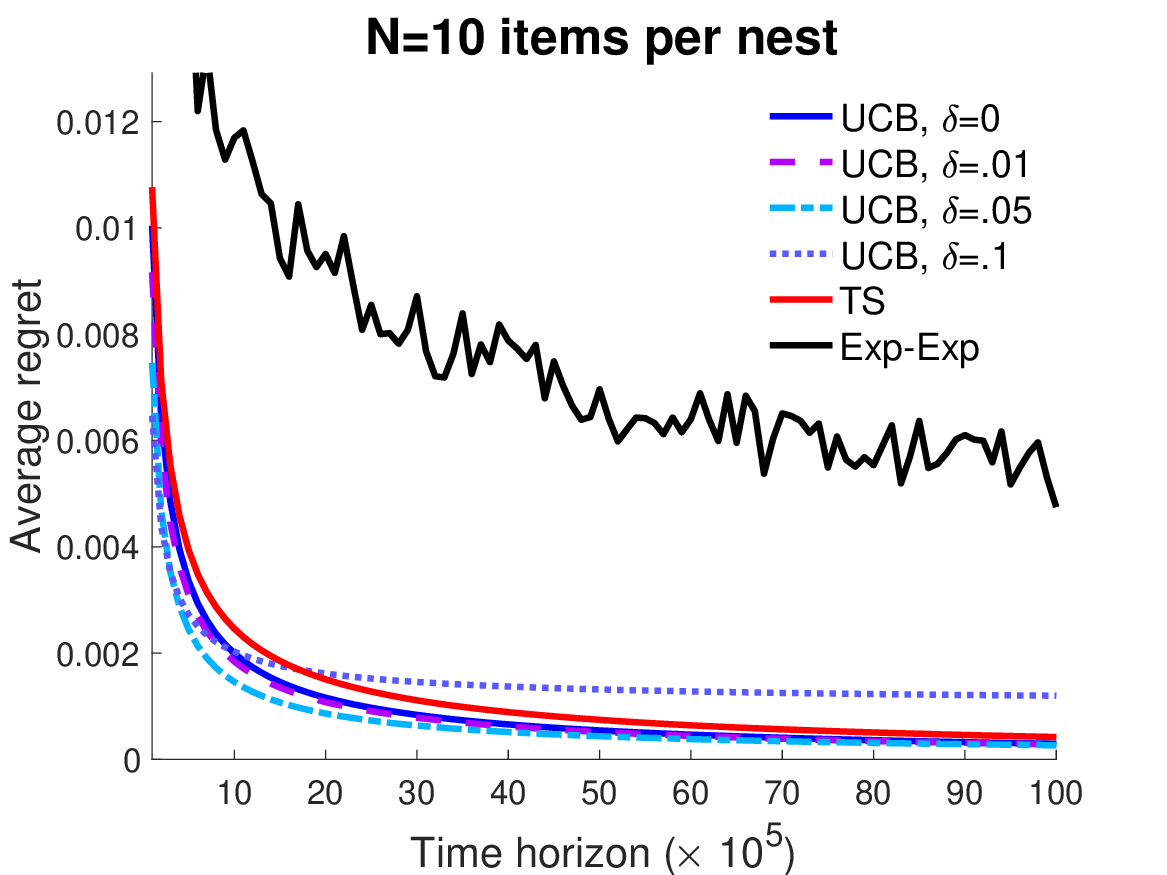}
\includegraphics[width=0.48\linewidth]{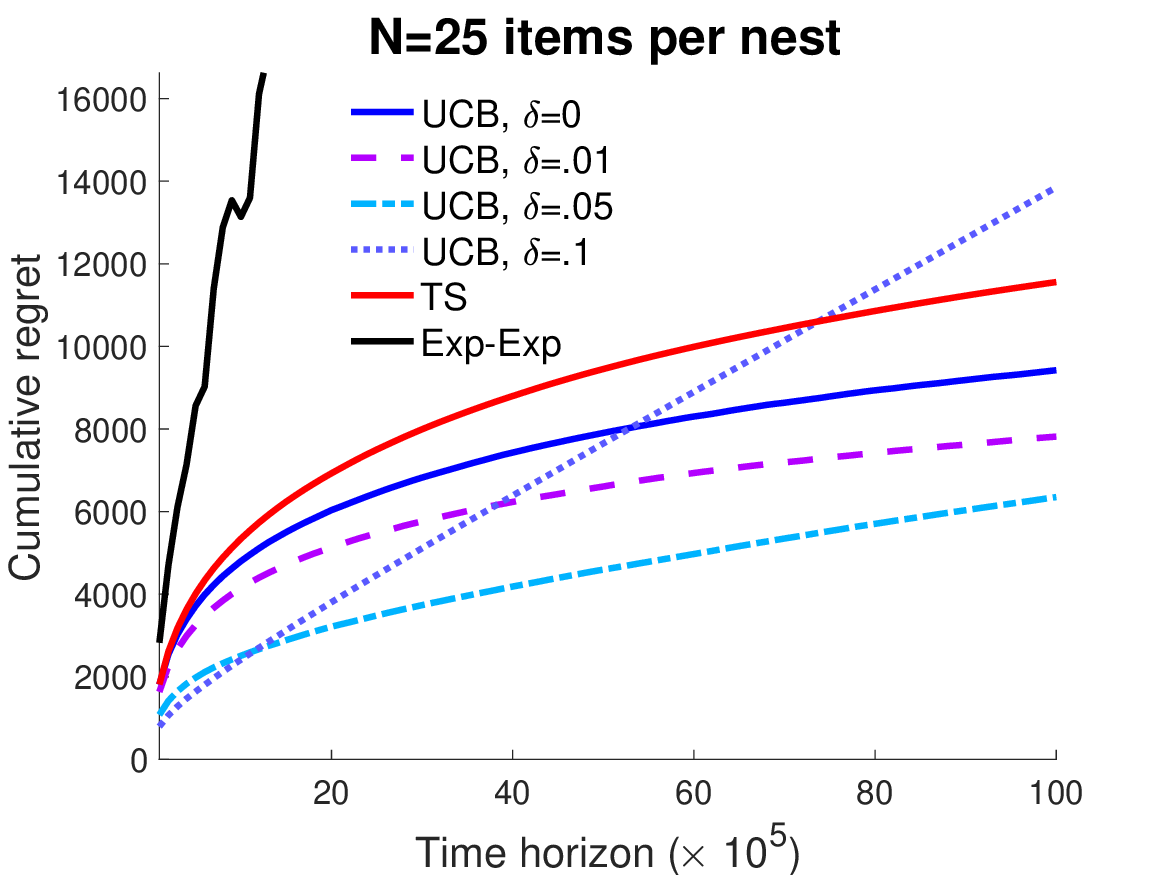}
\includegraphics[width=0.48\linewidth]{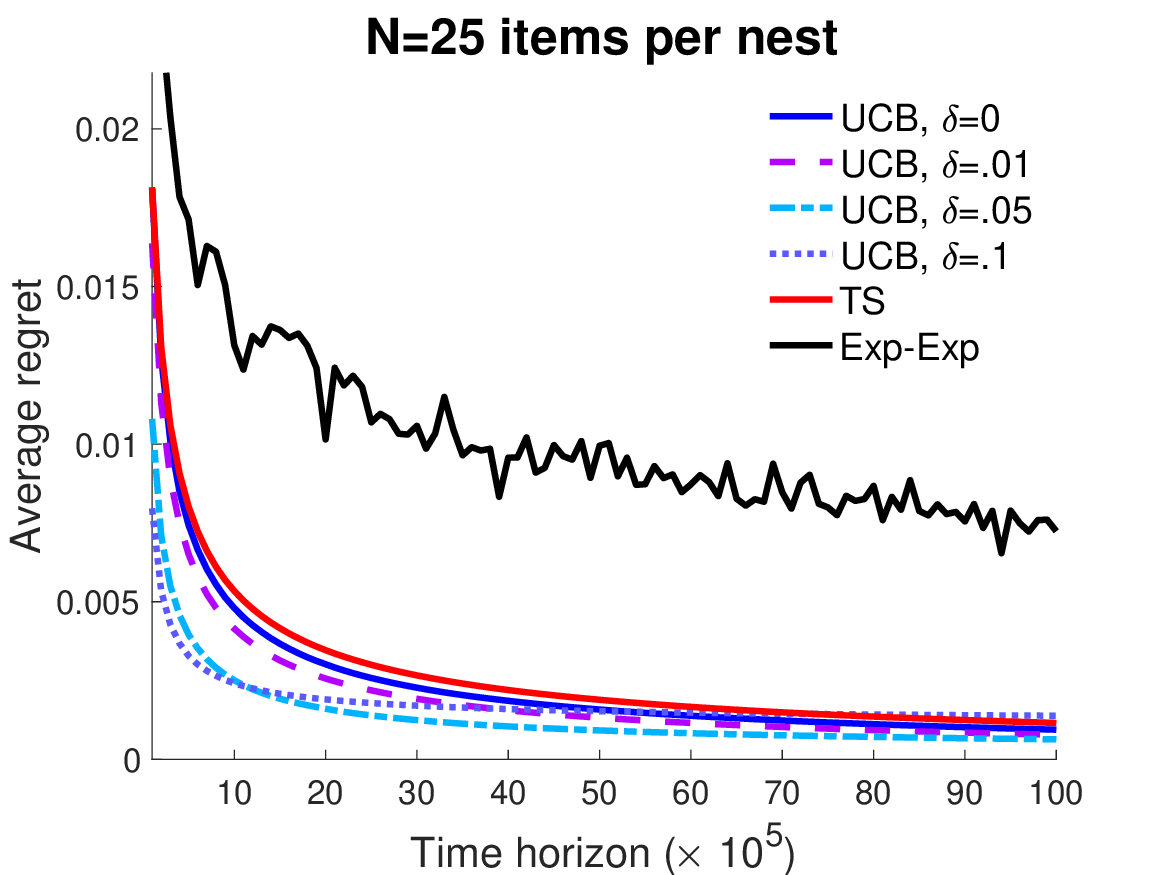}
\includegraphics[width=0.48\linewidth]{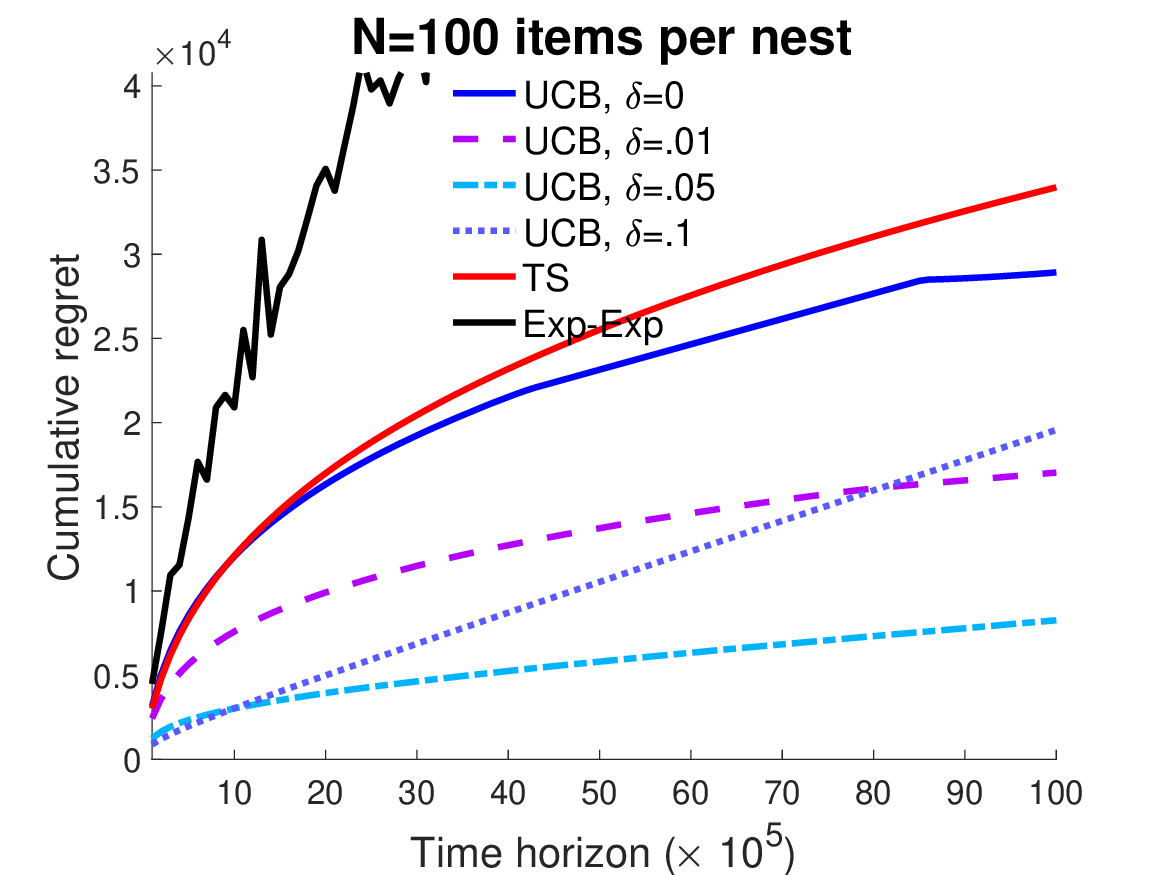}
\includegraphics[width=0.48\linewidth]{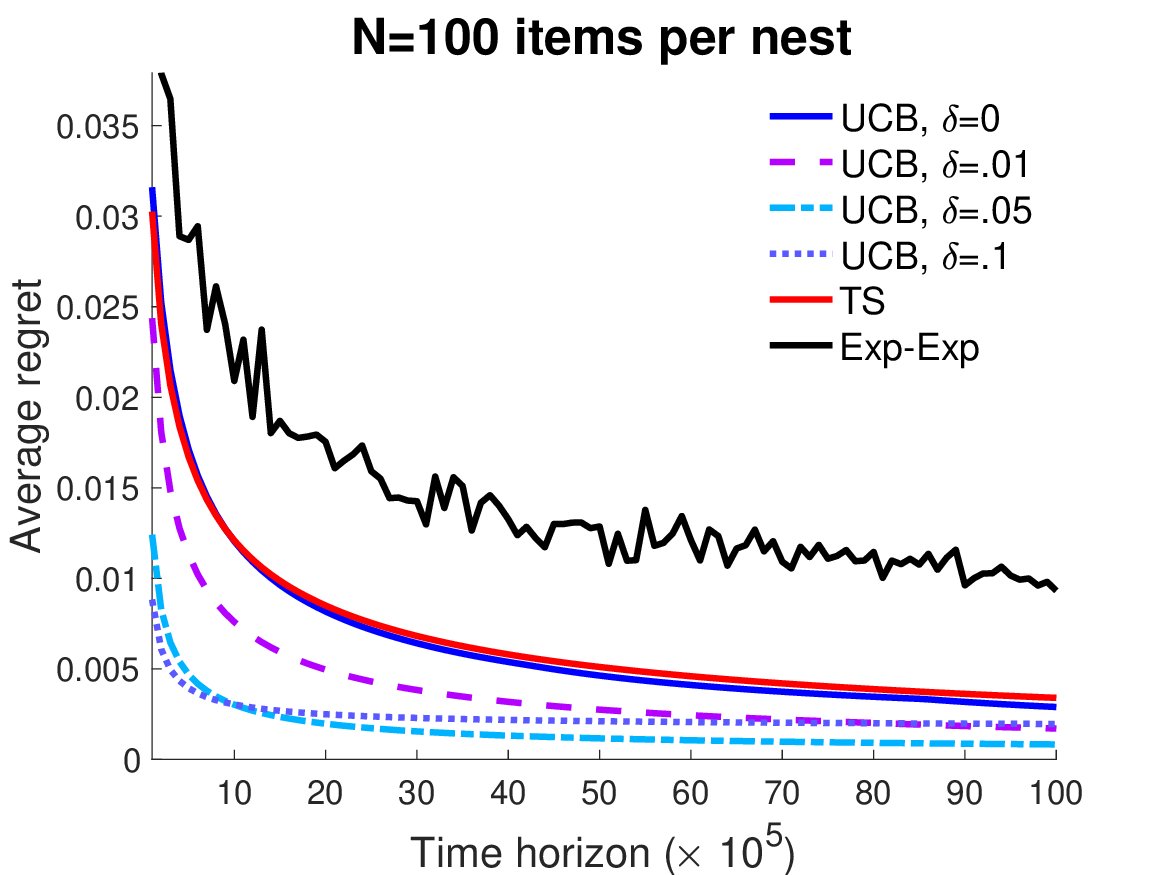}
\caption{Accumulated (left) and average (right) regret of our policy and competitive policies with $M=5$ nests, varying the number of items per nest $N$ and the granularity parameter $\delta$.
\textsf{TS} stands for the Thompson Sampling algorithm and \textsf{Exp-Exp} stands for the exploration-exploitation algorithm,
with details in the main text.}
\label{fig:trend}
\end{figure}

\section{Numerical results}
\label{sec:numerical}

We present  numerical studies of our proposed policies for dynamic nested assortment planning on synthetic data.
The main focus of our simulation is the regret of our policies under various model parameter settings of $M$, $N$, and $T$,
as well as the effect of the discretization granularity $\delta\in[0,1]$ on the regret.


	For each nest $i\in[M]$, the revenue parameters $\{r_{ij}\}_{j=1}^N$ are independently and identically distributed from the uniform distribution on $[0.2, 0.8]$
	and the preference parameters $\{v_{ij}\}_{j=1}^N$ are independently and identically distributed from the uniform distribution on $[10/N(M-1), 20/N(M-1)]$,
	where $N$ is the number of items in each nest.
	The nest discounting parameters $\{\gamma_i\}_{i=1}^M$ are independently and identically distributed from the uniform distribution on $[0.5,1]$.
	
	We consider the different combinations of parameters in terms of $M$ (the number of nests), $N$ (the number of items per nest), $T$ (time horizon length), and $\delta$ (the granularity parameter in the heuristic discretized policy). We note that  $\delta=0$ means 	that no discretization is carried out. 
	For each $(M,N)$ settings, we generate model parameters $\{r_{ij},v_{ij},\gamma_i\}_{i,j=1}^{M,N}$ as described in the previous paragraph,
	and then run the dynamic assortment policy for 100 independent trials.
	The median and maximum accumulated regret over $T$ periods  are reported.

In Table \ref{tab:main}, we compare the accumulated regret of our proposed policies with different granularity parameters $\delta$,
under a range of different parameter settings of number of nests $M$, number of items per nest $N$, and time horizon $T$. We also compare the performances of our algorithms with some competitive algorithm baselines, such as the Thompson Sampling (\textsf{TS}) algorithm and the Explore-then-Exploit (\textsf{Exp-Exp}) algorithm. {The \textsf{TS} algorithm works similar to our UCB algorithm (with $\delta = 0$), where the level-set assortments are first generated for each nest and the algorithm learns and optimizes the assortment while maintaining Beta priors for the aggregate parameters, as suggested in \citep{Agrawal17Thompson}. In the \textsf{Exp-Exp} algorithm, the level-set assortments are also first generated. However, the exploration phase (learning the aggregate parameters) and the exploitation phase ( exploiting the estimated optimal assortment) are separated.}

In Figure  \ref{fig:trend}, we further plot the accumulated and average regret of our policies for time horizon when $T$ is large ($T$ between $10^5$ and $10^7$).
 From both Table \ref{tab:main} and Figure  \ref{fig:trend}, one can see a clear pattern of sub-linear accumulated regret. Moreover, when $N$ is small as compared to $T$,  a smaller discretization granularity leads to better empirical performance; while when $N$ is large, a larger discretization granularity is better. We can observe from both Table \ref{tab:main} and Figure \ref{fig:trend} that our algorithm with the appropriate granularity level $\delta$ consistently outperforms
 the two baseline methods.
The \textsf{TS} algorithm performs similarly to our UCB algorithm with $\delta = 0$. It is also possible to combine the \textsf{TS} algorithm with our discretization heuristic (i.e., setting $\delta$ to be a positive value), and we would expect the similar performance as its UCB counterpart with the same discretization parameter $\delta$. For simplicity and interpretability of the figures, we omit those performance curves.

 We also remark that, because of the inherent instability of the \textsf{Exp-Exp} algorithm (as a consequence of the fact that
 \textsf{Exp-Exp} commits to a \emph{single} assortment for the majority of $T$ time periods),
 the curve for \textsf{Exp-Exp} displayed in Figure \ref{fig:trend} is much more wiggly compared to the other algorithms that are more stable
 with smaller variance.

{
One principle for choosing an appropriate value for $\delta$ is based on the time horizon $T$. When $T$ is larger, the inherent bias could result in a $\delta T$ cumulative regret that is linear with $T$, which is typically reflected by the UCB with $\delta = .1$ curves in all settings of Figure~\ref{fig:trend}. (For large enough $T$, a linear growth can also be observed for some of the UCB with $\delta = 0.05$ curves.) Therefore, in the long term, a smaller $\delta$ would reduce the negative impact to the cumulative regret. However, when $T$ is smaller, a larger $\delta$ means less amount of aggregate parameters to learn, and therefore the algorithm benefits from a quick start. This is clearly reflected in the $N = 25$ setting of Figure~\ref{fig:trend} where the UCB with $\delta = .1$ curve enjoys the lowest regret for small $T$, and gradually loses its advantages as $T$ increases. For future directions, it is a very interesting question to study how to appropriately (and maybe even dynamically) set the values of $\delta$ to derive a better theoretical regret bound.

We note that a linear growth of the regret with $T$ would only occur when the algorithm fails to recover the optimal assortment after the discretization process. On the other hand, if $\delta$ is set to be a small enough value such that the optimal assortment can be found even after the discretization process, such a linear growth of the regret will not occur. In Table~\ref{tab:percentage}, we report the percentage of instances in the corresponding settings of Figure~\ref{fig:trend}, where the optimal assortment can be recovered after discretization. We note that a small percentage value corresponds to a linear growth curve in Figure~\ref{fig:trend}.

\begin{table}[!t]

\centering
\caption{Percentage of the instances in the settings of Figure~\ref{fig:trend} where the optimal assortment can be recovered after discretization.} \label{tab:percentage}
\begin{tabular}{l|cccc}
\hline
			& $\quad\delta = 0\quad$	& $\quad\delta = .01\quad$	& $\quad\delta = .05\quad$	& $\quad\delta = .1\quad$\\
\hline
$N=10$		& $100\%$	& $99\%$			& $66\%$			& $36\%$	 \\
$N=25$		& $100\%$	& $99\%$			& $28\%$			& $2\%$	 \\
$N=100$		& $100\%$	& $99\%$			& $1\%$			& $0\%$	\\
\hline
\end{tabular}
\end{table}
}

We also remark that, when $N$ is small, the gap between two level-set assortments in each nest is potentially large and therefore the bias resulting
from a large $\delta$ value could also be large.
This means that when $N$ is small, giving rise to only a few ``level-set'' assortments,
$\delta$ cannot be set too large because otherwise the optimal level-set assortment might be missed
because of the large gap between integer multiples of $\delta$s.
On the other hand, when $N$ is large, even if $\delta$ is big the potential bias introduced by discretization could still be smaller,
because there might be a ``level-set'' assortment close to every levels of $i\delta$, $i\in\mathbb N$, at the discretization level of $\delta$.
This means that discretization at a larger value of $\delta$ is potentially more beneficial when $N$ is large,
because little additional bias is increased but the number of ``level-set'' assortments to be considered is significantly fewer when $\delta$ is large.

\begin{figure}[!t]
\centering
\includegraphics[width=0.48\linewidth]{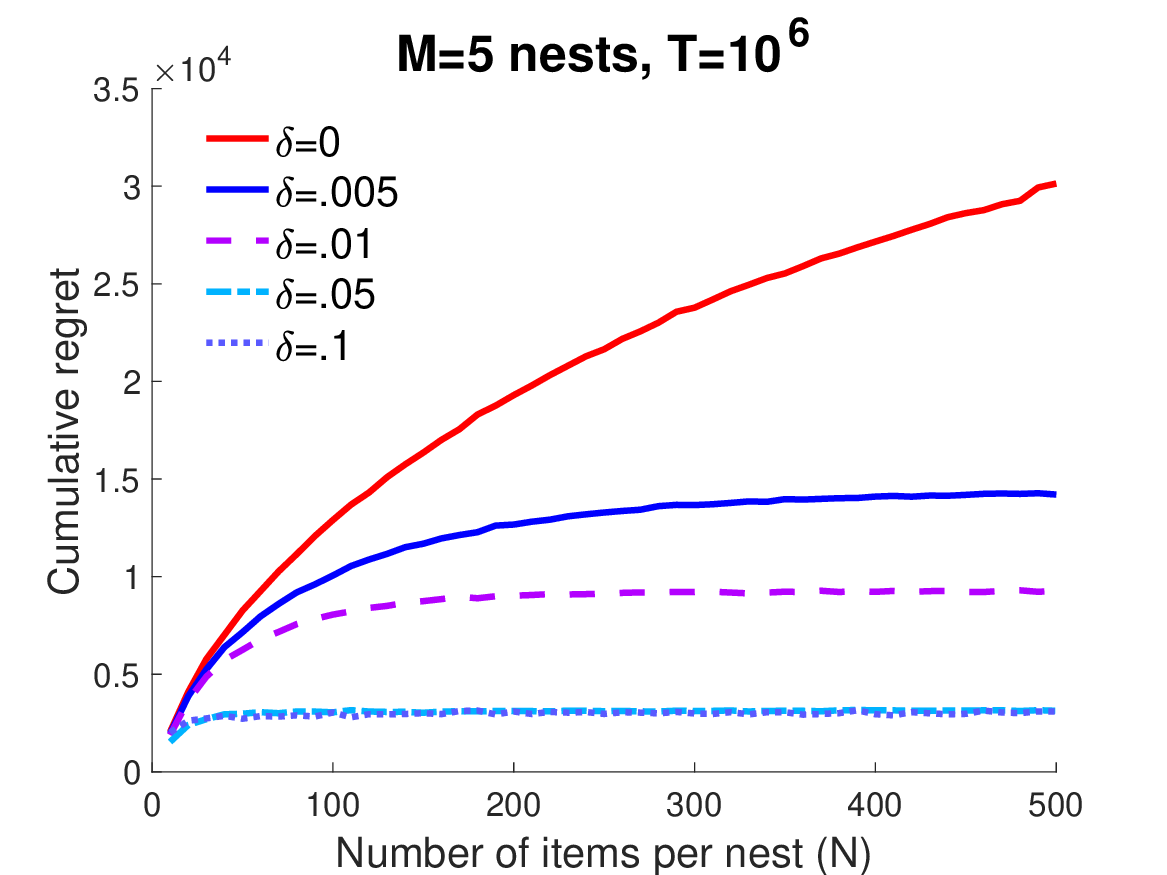}
\includegraphics[width=0.48\linewidth]{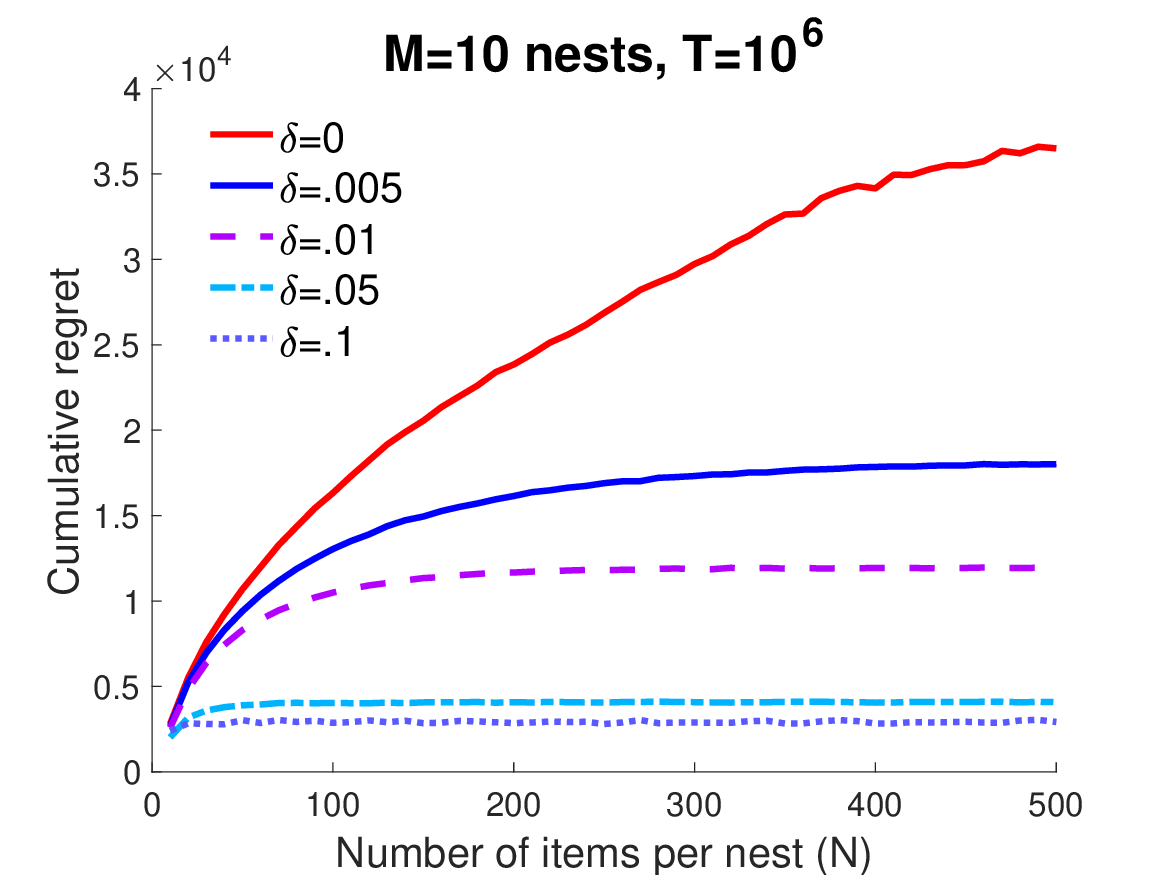}
\caption{Cumulative regret of our algorithm with varying number of products per nest ($N$),
at different levels of discretization granularity ($\delta$).}
\label{fig:varn}
\end{figure}

{
Finally, in Figure \ref{fig:varn} we compare the cumulative regret of our proposed UCB algorithm (at different levels of discretization granularity $\delta$)
by holding $M,T$ fixed and varying the number of products per nest ($N$).
As we observe from Figure \ref{fig:varn}, when $\delta=0$ (i.e., no discretization carried out),
the cumulative regret of our algorithm does scale on the order of $O(\sqrt{N})$ with $N$, suggesting that our upper bound results in
Theorem \ref{thm:ucb} and Corollary \ref{cor:regret-preliminary} are tight.
Figure \ref{fig:varn} also shows that with larger discretization granularity level $\delta$,
the regret of the proposed UCB algorithm scales more mildly with increasing number of products per nest $N$.
}

{
\subsection{Experiments following the setting in \cite{davis2014assortment}}\label{sec:davis-experiment}

In this subsection, we report the simulation results on a set of different classes of the synthetic problem instances. The synthetic problem instances are generated similar as described in \citep{davis2014assortment}. The instance is parameterized by $\epsilon \in (0, 1)$. For each nest $i \in [M]$, we first generate the nest discounting parameter $\gamma_i$ independently from the uniform distribution $[0.5, 1]$. We then generate the first $(N - 1)$ items as follows. For each $j \in [N - 1]$, we independently sample $U_{ij}$   from the uniform distribution over $[0, 4]$, $X_{ij}$  from the uniform distribution over $[0.1, 1]$, and $Y_{ij}$  from the uniform distribution over $[0.01, 0.1]$, and set $r_{ij} = \epsilon^{U_{ij}} \cdot X_{ij}$ and $v_{ij} = \epsilon^{2-U_{ij}} \cdot Y_{ij}$. Finally, we let $r_{iN} = 0$ and $v_{iN}= \epsilon^{-1} \cdot Y_{iN}$ where $Y_{iN}$ is also independently and uniformly sampled from $[0.01, 0.1]$. We note that the main differences between our generating procedure and that of \citep{davis2014assortment} are that the range of $X_{ij}$ is $[1, 10]$ and the range of $Y_{ij}$ is $[0.2, 1.8]$ in  \citep{davis2014assortment}. While the differences for $X_{ij}$ only affect the revenue parameters $\{r_{ij}\}$ (and therefore the revenues of all candidate assortments) up to a scaling factor, the differences for $Y_{ij}$ are because that weight of the no-purchase option is set to $10$ in  \citep{davis2014assortment}, but normalized to $1$ in our paper. Considering the typical value of the discounting parameter $\gamma_i$ (which is $\sim 0.75$), we therefore correspondingly reduce the range of $Y_{ij}$ by $\sim 10^{1/0.75} \approx 20$ times for the normalization purpose. We set $M = 5$ and pick $(\epsilon, N)$ from $\{0.6, 0.4\} \times \{25, 100\}$ to generate four classes of the problem instances. 

In Figure \ref{fig:davis-experiment}, we report the performance of our algorithms and the two baseline algorithms \textsf{TS} and \textsf{Exp-Exp} in these four settings. As one can observe from Figure \ref{fig:davis-experiment}, the comparison between our algorithms and the baseline algorithms are similar to the settings in Figure~\ref{fig:trend}, with one difference that the \textsf{TS} algorithm performs better than our UCB algorithm with $\delta = 0$ (and sometimes our UCB algorithm with $\delta = 0.01$). 

\begin{figure}[t]
\centering
\includegraphics[width=0.48\linewidth]{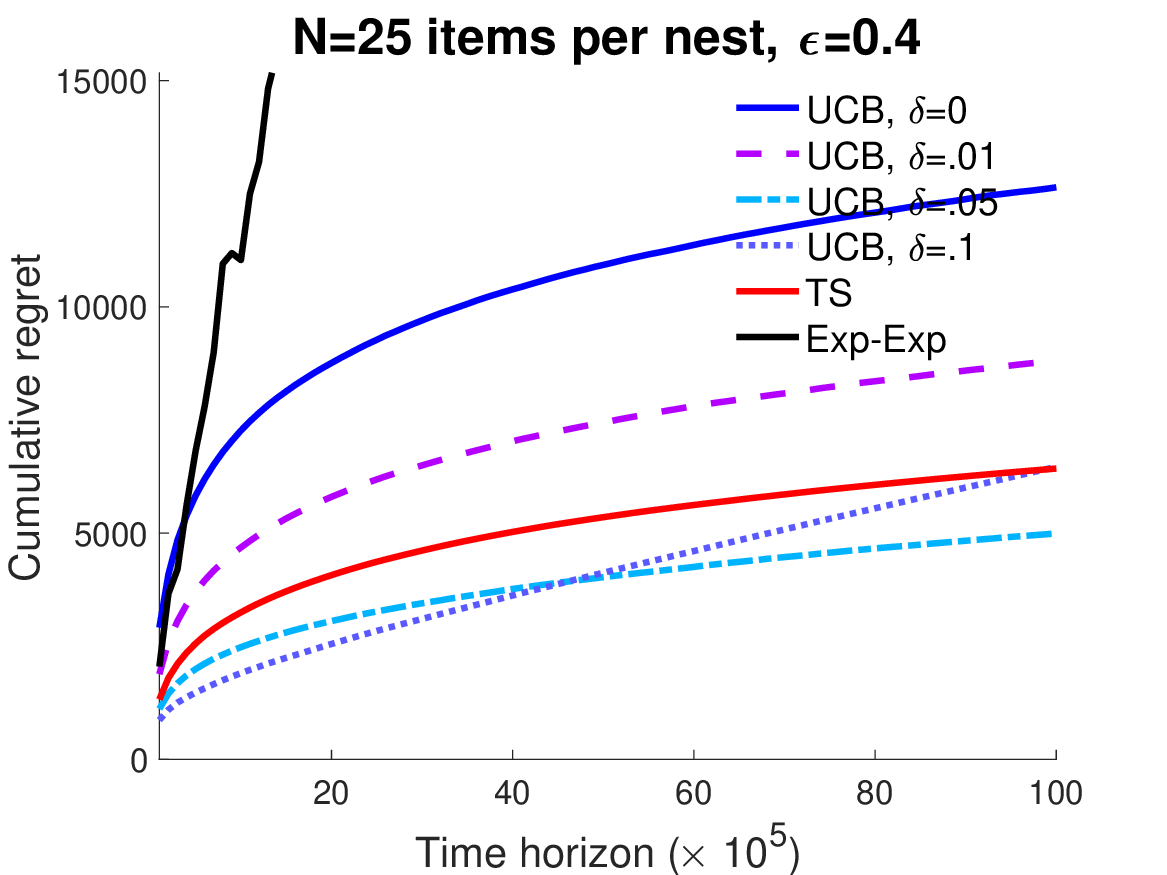}
\includegraphics[width=0.48\linewidth]{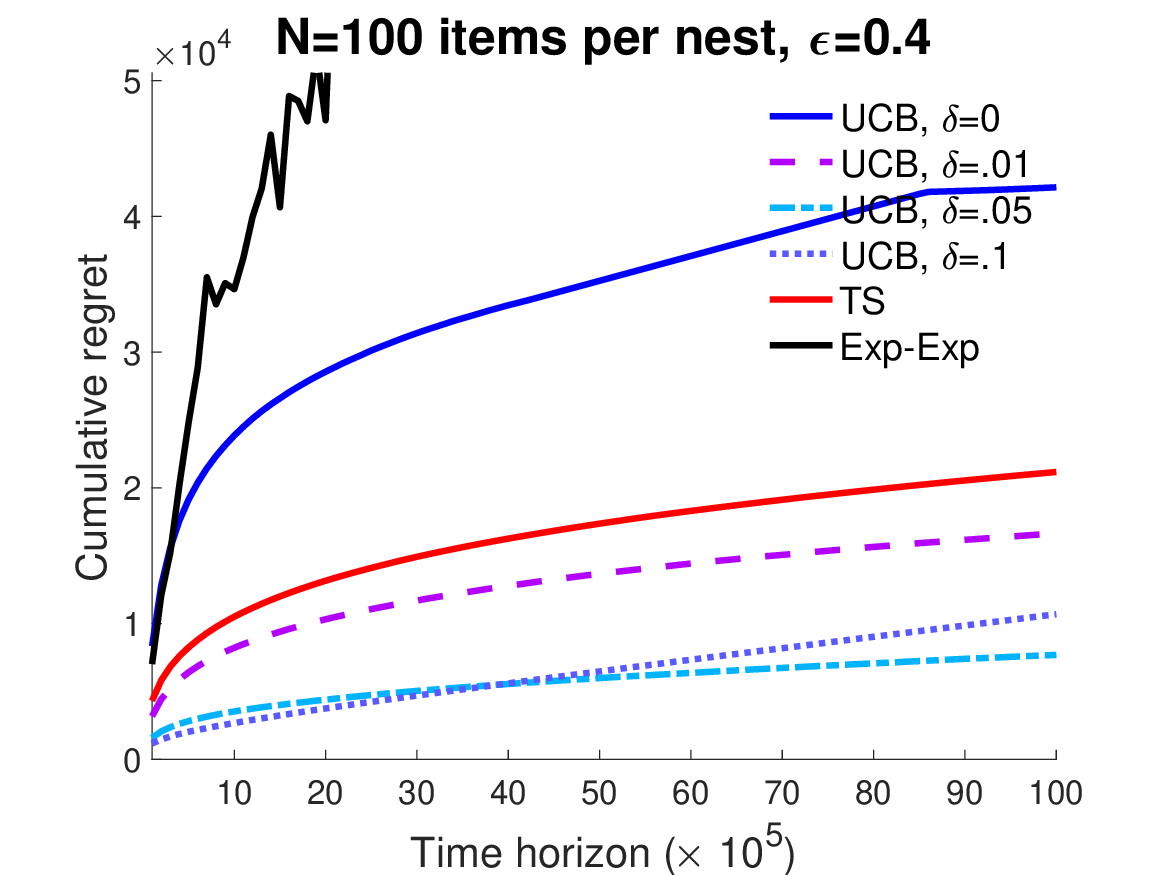}
\includegraphics[width=0.48\linewidth]{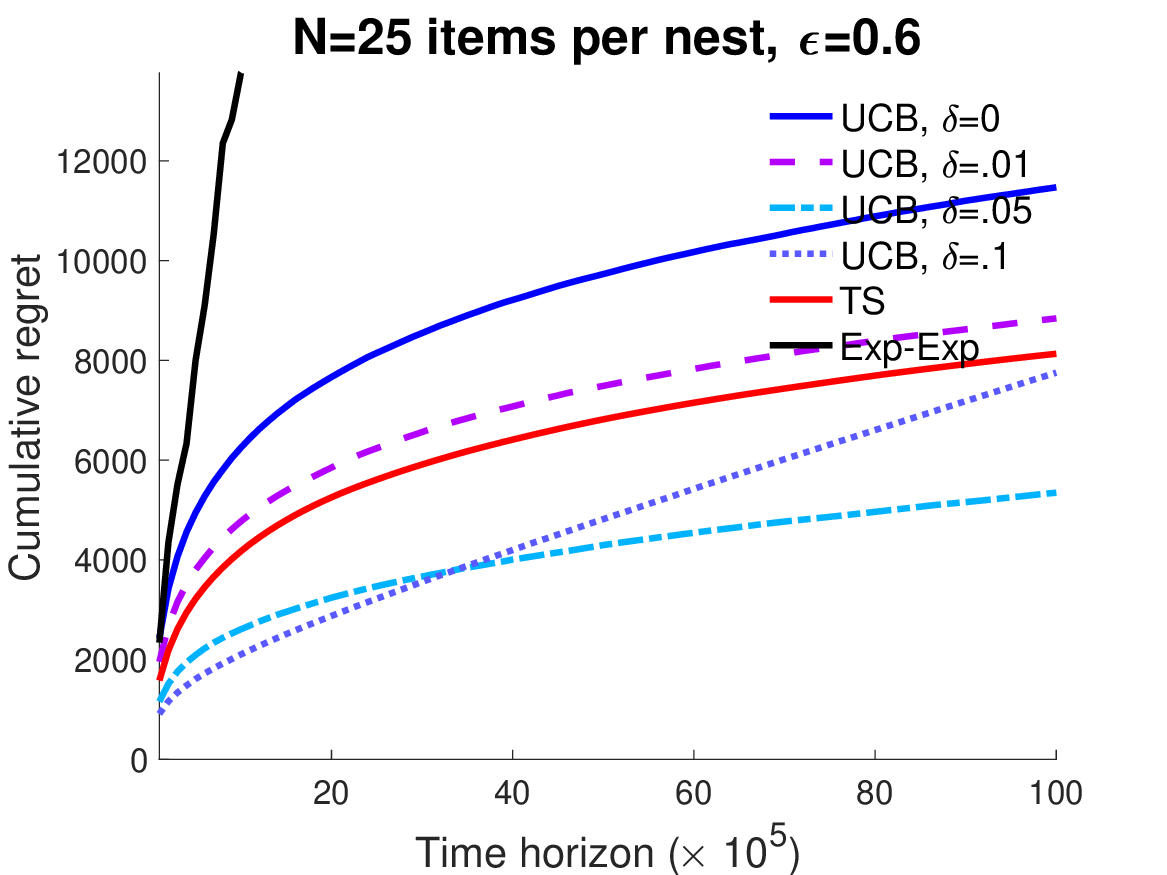}
\includegraphics[width=0.48\linewidth]{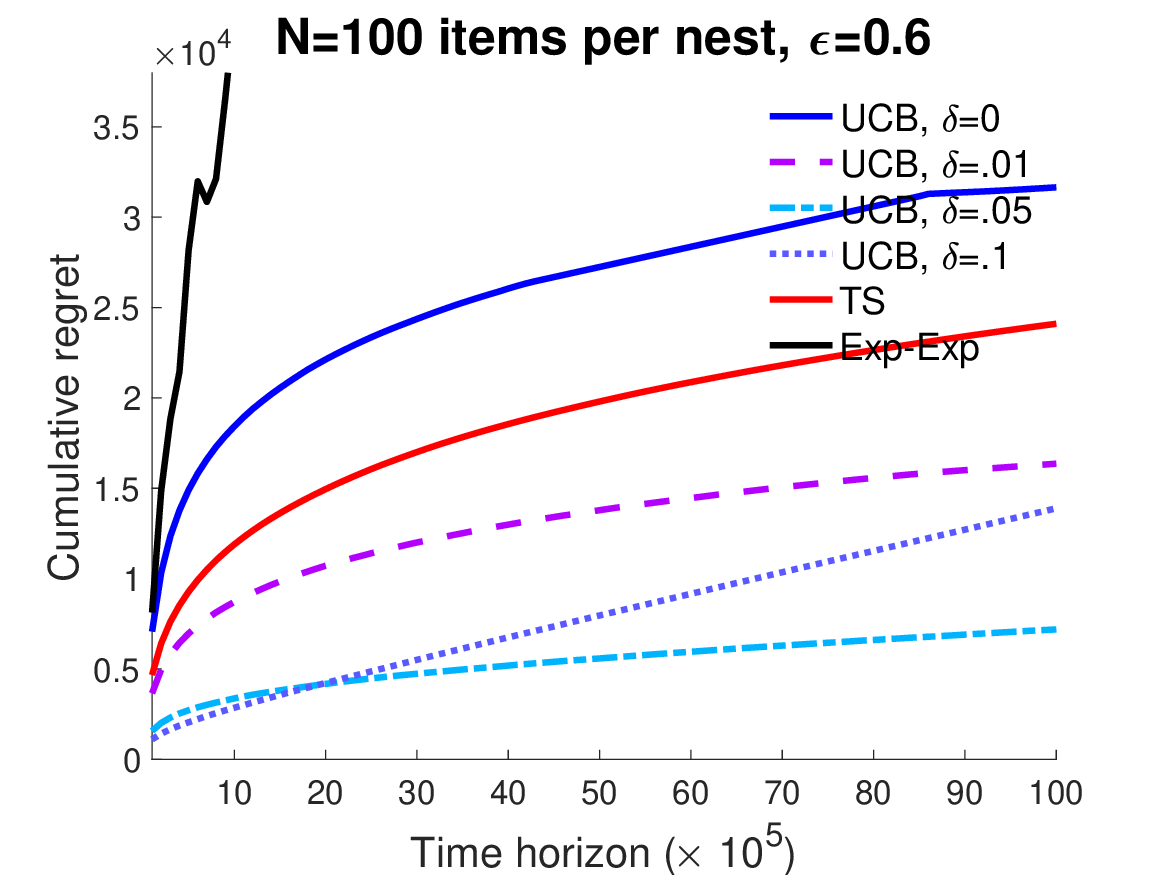}
\caption{Cumulative regret of our algorithm and other comparative methods for experiments outlined in Sec.~\ref{sec:davis-experiment}.
$\epsilon$ is a parameter used in generating problem instances, which is described in further details in the main text.}
\label{fig:davis-experiment}
\end{figure}

}

\section{Conclusions}
\label{sec:conclusions}
	
In this paper, we consider the dynamic assortment planning problem under the nested logit models and we propose the UCB policy to achieve $\tilde{O}(\sqrt{MNT} + MN^2)$ accumulative regret. We also propose the discretization heuristic that shows the improved empirical performance.


	
There are several interesting future directions of the current work. The first technical problem is to investigate the dependency of $N$ in the lower bound. 
Second, it is interesting to further extend the current two-level nested logit models to variants of nested models (e.g., constrained nested logit models \citep{Gallego2014}, $d$-level nested logit models \citep{Li2015d}). Third, the dynamic assortment planning is a relatively new topic in revenue management and the understanding of this problem is still limited. Therefore, most existing work (including this paper) focuses on the stylized models where the assortment is the only decision variable. One potential future work is to  incorporate other operational decisions and constraints, such as prices and inventory constraints.

\section*{Acknowledgment}
We would like to thank the Department Editor, the Senior Editor, and the two referees for their constructive suggestions that greatly improve the paper.  Xi Chen is supported by the NSF Grant via IIS-1845444. Chao Shi would like to thank the National Natural Science Foundation of China (Grant 71501043).

\bibliographystyle{informs2014} 
\bibliography{refs} 

\ECSwitch


\ECHead{Proofs of Statements}
	
\section{Proof of Lemma \ref{lem:negative-binomial}}
\begin{repeatlemma}[Lemma \ref{lem:negative-binomial} (restated).]
		For each epoch $\mathcal E_\tau$ and nest $i\in [M]$, let $\hat\theta_i\in\mathcal K_i$.
		The expectations of the number of iterations and total revenues collected in which nest $i$ is purchased (denoted by $\hat n_{i,\tau}$ and $\hat r_{i,\tau}$ respectively in Algorithm \ref{alg:ucb}) satisfies
		the following \emph{regardless of the other offered assortments $\hat\theta_{i'}$ for $i'\neq i$}:
		\begin{enumerate}
			\item $\mathbb E[\hat n_{i,\tau}] = u_{i,\hat\theta_i}$;
			\item $\mathbb E[\hat r_{i,\tau}|\hat n_{i,\tau}] = \hat n_{i,\tau}\phi_{i,\hat\theta_i}$;
		\end{enumerate}
\end{repeatlemma}
\proof{Proof of Lemma \ref{lem:negative-binomial}}
	Simple calculations show that (see for example Corollary A.1 of \cite{Agrawal16MNLBandit})
	\begin{equation}
	\Pr\left[\hat n_{i,\tau}=k\right] = \left(\frac{u_{i,\hat\theta_i}}{1+u_{i,\hat\theta_i}}\right)^k\left(\frac{1}{1+u_{i,\hat\theta_i}}\right) \;\;\;\;\;\;\;\text{for}\;\; k=0,1,2,\cdots
	\end{equation}
	That is, $\hat n_{i,\tau}$ is a geometric random variable with parameter $1/(1+u_{i,\hat\theta_i})$.
	Hence, $\hat n_{i,\tau}$ is an \emph{unbiased} estimator of $u_{i,\hat\theta_i}$, meaning that $\mathbb E\hat n_{i,\tau}= u_{i,\tau}$.

	The distribution and expectation of $\hat r_{i,\tau}$ can be similarly derived, using the property that $\mathbb E[\hat r_{i,\tau}|i_t=i] = \phi_{i,\hat\theta_i}$.
	
\section{Proof of Lemma \ref{lem:binary-search}}
	\begin{repeatlemma}[Lemma \ref{lem:binary-search} (restated).]
	The following hold for all $\lambda\in[0,1]$:
	\begin{enumerate}
	\item If $\bar R^*\geq\lambda$ then there exists $\vct\theta\in\mathcal K_1\times\cdots\times\mathcal K_M$ such that $\psi_\lambda(\vct\theta)\geq\lambda$;
	furthermore if $\bar R^*>\lambda$ then the inequality is strict;
	\item If $\bar R^*\leq\lambda$ then for all $\vct\theta\in\mathcal K_1\times\cdots\times\mathcal K_M$, $\psi_\lambda(\vct\theta)\leq\lambda$;
	furthermore if $\bar R^*<\lambda$ then the inequalities are strict.
	\end{enumerate}
	\end{repeatlemma}
	\proof{Proof of Lemma \ref{lem:binary-search}}
	Let $\vct\theta^*=(\theta_1^*,\cdots,\theta_M^*)\in\mathcal K_1\times\cdots\times\mathcal K_M$ be a maximizer of $\bar R'$ (i.e., $\bar R^*=\bar R'(\vct\theta^*)$).
	By definition, $\sum_{i=1}^M(\bar\phi_{i,\theta_i^*}-\bar R^*)\bar u_{i,\theta_i^*} =\bar R^*$.
	If $\bar R^*\geq \lambda$, then $\sum_{i=1}^M(\bar\phi_{i,\theta_i^*}-\lambda)\bar u_{i,\theta_i^*}\geq \sum_{i=1}^M(\bar\phi_{i,\theta_i^*}-\bar R^*)\bar u_{i,\theta_i^*} = \bar R^*\geq\lambda$.
	Therefore $\psi_{\lambda}(\vct\theta^*)\geq\lambda$.
	Furthermore, if $\bar R^*>\lambda$ then the last inequality in the chain of inequalities is strict.
	The first property is thus proved.
	
	We next prove the second property.
	Assume by way of contradiction that there exists $\vct\theta=(\theta_1,\cdots,\theta_M)\in\mathcal K_1\times\cdots\times\mathcal K_M$ such that $\psi_\lambda(\vct\theta)>\lambda$,
	meaning that $\sum_{i=1}^M(\bar\phi_{i,\theta_i}-\lambda)\bar u_{i,\theta_i} > \lambda$.
	Re-arranging terms and dividing both sides by $(1+\sum_{i=1}^M\bar u_{i,\theta_i})$ we have $\bar R'(\vct\theta) = [\sum_{i=1}^M\bar\phi_{i,\theta_i}\bar u_{i,\theta_i}]/[1+\sum_{i=1}^M\bar u_{i,\theta_i}] > \lambda$.
	This contradicts the assumption that $\bar R^*=\max_{\vct\theta\in\mathcal K_1\times\cdots\times\mathcal K_m} R'(\vct\theta) \leq \lambda$.
	To prove the second half of the second property, simply replace all occurrences of $>$ by $\geq$.
	
\section{Proof of Lemma \ref{lem:new-nested-concentration}}
\begin{repeatlemma}[Lemma \ref{lem:new-nested-concentration} (restated).]
		Suppose $T(i,\theta)\geq 96\ln(2MTK)$.
		With probability $1-T^{-1}$ uniformly over all $i\in[M]$, $\theta\in\mathcal K_i$ and $t\in[T]$
		\begin{align}
		\big|\hat u_{i,\theta}-u_{i,\theta}\big| &\leq\min\left\{U, 3\sqrt{\frac{48\max(\hat u_{i,\theta},\hat u_{i,\theta}^2)\ln(2MTK)}{T(i,\theta)}} +\frac{144\ln(2MTK)}{T(i,\theta)} \right\};\label{seq:u_ci}\\
		\big|\hat\phi_{i,\theta}-\phi_{i,\theta}\big|&\leq\min\left\{1, \sqrt{\frac{\ln(2MTK)}{T(i,\theta)\hat u_{i,\theta}}}\right\}.\label{seq:phi_ci}
		\end{align}
		In addition, if $u_{i,\theta}\geq 1$ then $\hat u_{i,\theta}\in[0.5u_{i,\theta},2u_{i,\theta}]$.
\end{repeatlemma}
\proof{Proof of Lemma \ref{lem:new-nested-concentration}}
We first prove the upper bound on $|\hat u_{i,\theta}-u_{i,\theta}|$ for fixed $i\in[M]$ and $\theta\in\mathcal K_i$.

\paragraph{Case 1: $u_{i,\theta}\leq 1$.} Let $\delta>0$ be a parameter to be specified later.
{ Applying Lemma \ref{lem:geometric} and noting that $u_{i,\theta}\leq 1$ implies $\mu\leq 1$ and $(1+\mu)^2\leq 4$
in Lemma \ref{lem:geometric},}
 we have
 {
\begin{align*}
\Pr\left[|\hat u_{i,\theta}-u_{i,\theta}| > \delta u_{i,\theta}\right]
&\leq \exp\left\{-\frac{ n u_{i,\theta}\delta^2}{2(1+\delta)\times 4}\right\} + \exp\left\{-\frac{n u_{i,\theta}\delta^2}{6\times 4}\left(3-\frac{2\delta u_{i,\theta}}{1+u_{i,\theta}}\right)\right\}\\
&\leq \exp\left\{-\frac{nu_{i,\theta}\delta^2}{16\max(1,\delta)}\right\} + \exp\left\{-\frac{nu_{i,\theta}\delta^2}{24}(3-2\delta u_{i,\theta})\right\}.
\end{align*}
}

Suppose in addition that {$\delta u_{i,\theta}\leq 1$}. Then
\begin{equation*}
\Pr\left[|\hat u_{i,\theta}-u_{i,\theta}| > \delta u_{i,\theta}\right]
\leq 2\exp\left\{-\frac{nu_{i,\theta}\delta^2}{24}\right\}.
\end{equation*}

Equating the right-hand side of the above inequality with $1/MKT^2$ we have,
\begin{equation*}
\delta = \sqrt{\frac{48\ln(2MTK)}{u_{i,\theta}T_{i,\theta}}}.
\end{equation*}
 Applying the union bound over all $i\in[M]$, $\theta\in\mathcal K_i$ and $t\in[T]$, we have with probability $1-T^{-1}$ that
\begin{equation}
\big|\hat u_{i,\theta}-u_{i,\theta}\big| \leq \delta u_{i,\theta} \leq \sqrt{\frac{48u_{i,\theta}\ln(2MTK)}{T(i,\theta)}}. 
\label{eq:u-true-ci-part1}
\end{equation}

{Note that if $T(i,\theta)\geq 48\ln(2MTK)$ the condition  $\delta u_{i,\theta}\leq 1$ are met.}
Replacing all occurrences of $u_{i,\theta}$ in Eq.~(\ref{eq:u-true-ci-part1}) by $\hat u_{i,\theta}$
and using the fact that $\sqrt{a+b}\leq \sqrt{a}+\sqrt{b}$, we have
{
\begin{align}
\big|\hat u_{i,\theta}-u_{i,\theta}\big|
&\leq  \sqrt{\frac{48\hat u_{i,\theta}\ln(2MTK)}{T(i,\theta)}} + \sqrt{\frac{48|\hat u_{i,\theta}-u_{i,\theta}|\ln(2MTK)}{T(i,\theta)}}\nonumber\\
&\leq \sqrt{\frac{48\hat u_{i,\theta}\ln(2MTK)}{T(i,\theta)}} + \sqrt{\frac{48\ln(2MTK)}{T(i,\theta)}}\times\sqrt{\frac{48 u_{i,\theta}\ln(2MTK)}{T(i,\theta)}} \nonumber\\
&\leq \sqrt{\frac{48\hat u_{i,\theta}\ln(2MTK)}{T(i,\theta)}} + \frac{48\ln(2MTK)}{T(i,\theta)},
\label{eq:u-ci-part1}
\end{align}
where the last inequality holds because $u_{i,\theta}\leq 1$.
}

\paragraph{Case 2: $u_{i,\theta}>1$.}
Let $\delta\in(0,1]$ be a parameter to be specified later.
{Applying Lemma \ref{lem:geometric} and noting that $u_{i,\theta}>1$ implies $\mu\geq 1$ and $(1+\mu)^2\leq 4\mu^2$
in Lemma \ref{lem:geometric},}
 we have
{
\begin{align*}
\Pr\left[|\hat u_{i,\theta}-u_{i,\theta}| > \delta u_{i,\theta}\right]
&\leq \exp\left\{-\frac{T(i,\theta)u_{i,\theta}^2\delta^2}{6\times 4u_{i,\theta}^2}\left(3-\frac{2u_{i,\theta}\delta}{1+u_{i,\theta}}\right)\right\} + \exp\left\{-\frac{T(i,\theta)u_{i,\theta}^2\delta^2}{2\times 4u_{i,\theta}^2}\right\}\\
&\leq \exp\left\{-\frac{T(i,\theta)\delta^2}{24}\left(3-2\delta\right)\right\} +  \exp\left\{-\frac{T(i,\theta)\delta^2}{8}\right\}\\
&\leq 2\exp\{-T(i,\theta)\delta^2/12\},
\end{align*}
where the last inequality holds under the condition that $\delta\leq 0.5$, which we will justify later in the proof.
}

Equating the right-hand side of the above inequality with $1/MKT^2$ we have
\begin{equation*}
\delta = \sqrt{\frac{24\ln(2MTK)}{T(i,\theta)}}
\end{equation*}
and applying the union bound over all $i\in[M]$, $\theta\in\mathcal K_i$ and $t\in[T]$, with probability $1-T^{-1}$,
\begin{equation}
\big|\hat u_{i,\theta}-u_{i,\theta}\big| \leq \delta u_{i,\theta} \leq \sqrt{\frac{24 u_{i,\theta}^2\ln(2MTK)}{T(i,\theta)}}.
\label{eq:u-true-ci-part2}
\end{equation}

Note that $\delta\leq 0.5$ holds if $T(i,\theta)\geq 96\ln(2MTK)$.
In addition, if $T(i,\theta)\geq 96\ln(2MTK)$ we have $|\hat u_{i,\theta}-u_{i,\theta}|\leq 0.5u_{i,\theta}$ and hence
 $\hat u_{i,\theta}\geq 0.5u_{i,\theta}$. Subsequently, Eq.~(\ref{eq:u-true-ci-part2}) implies
 \begin{equation}
 \big|\hat u_{i,\theta}-u_{i,\theta}\big| \leq \sqrt{\frac{96\hat u_{i,\theta}^2\ln(2MTK)}{T(i,\theta)}}.
 \label{eq:u-ci-part2}
 \end{equation}

Finally, combining Eqs.~(\ref{eq:u-ci-part1},\ref{eq:u-ci-part2}) we proved the upper bound on $|\hat u_{i,\theta}-u_{i,\theta}|$.

We next prove the upper bound on $|\hat\phi_{i,\theta}-\phi_{i,\theta}|$.
Recall that for each $\tau\in\mathcal T(i,\theta)$, $\hat r_{i\tau}$ is the sum of $\hat n_{i\tau}$ i.i.d.~random variables with mean $\phi_{i,\theta}$ and within range
$[0,1]$ almost surely.
Also note that $\sum_{\tau'\in\mathcal T(i,\theta)}\hat n_{i,\tau'} = T(i,\theta)\hat u_{i,\theta}$.
Applying Hoeffding's inequality (Lemma \ref{lem:hoeffding}) we have for any $\delta>0$ that
\begin{equation*}
\Pr\left[\big|\hat\phi_{i,\theta}-\phi_{i,\theta}\big| >\delta\right] \leq 2\exp\left\{-2\delta^2\cdot T(i,\theta)\hat u_{i,\theta}\right\}.
\end{equation*}

Equating the right-hand side of the above inequality with $1/M(K+1)T^2$ and applying the union bound,
we have with probability $1-T^{-1}$ uniformly over $i\in[M]$, $\theta\in\mathcal K_i$ and $t\in[T]$ that
\begin{equation}
\big|\hat\phi_{i,\theta}-\phi_{i,\theta}\big| \leq \sqrt{\frac{\ln(2MTK)}{T(i,\theta)\hat u_{i,\theta}}}.
\end{equation}

\section{Proof of Corollary \ref{cor:Rub}}
	\begin{repeatcorollary}[Corollary \ref{cor:Rub} (restated).]
	With probability $1-T^{-1}$, $\bar R'(\hat{\vct\theta})\geq R(\hat{\vct\theta})$ and $\bar R'(\vct\theta^*)\geq R'(\vct\theta^*)$,
	where $\hat{\vct\theta},\vct\theta^*\in\mathcal K_1\times\cdots\times\mathcal K_M$ are maximizers of $\bar R'$ and $R'$, respectively.
	\end{repeatcorollary}
	\proof{Proof of Corollary \ref{cor:Rub}}
	We first prove $\bar R'(\hat{\vct\theta})\geq R'(\hat{\vct\theta})$.
	By definition, $\sum_{i=1}^M(\bar\phi_{i,\hat\theta_i}-\bar R'(\hat{\vct\theta}))\bar u_{i,\hat\theta_i}=\bar R'(\hat{\vct\theta})$.
	In addition, because $\bar R'(\hat{\vct\theta})$ is the maximizer of $\bar R'$,
	setting $\lambda = \bar R'(\hat{\vct\theta})$ and by the second property of Lemma \ref{lem:binary-search} we know that
	$\psi_\lambda(\hat{\vct\theta})=\lambda$ and
	$\psi_\lambda(\vct\theta)\leq\lambda$
	for all $\vct\theta\in\mathcal K_1\times\cdots\times\mathcal K_M$, where $\psi_\lambda(\vct\theta)=\sum_{i=1}^M(\bar\phi_{i,\theta_i}-\lambda)\bar u_{i,\theta_i}$.
	
	We claim that $\bar\phi_{i,\hat\theta_i}\geq\bar R'(\vct\theta)$ whenever $\bar u_{i,\hat\theta_i}>0$.
	Assume the contrary, that $\bar\phi_{i,\hat\theta_i} < \bar R'(\vct\theta)=\lambda$ and $\bar u_{i,\hat\theta_i}>0$ for some $i\in[M]$.
	Consider $\hat{\vct\theta}'=(\hat\theta_1',\cdots,\hat\theta_M')$ defined as $\hat\theta_i'=\infty$ and $\hat\theta_{i'}'=\hat\theta_{i'}$ for all $i'\neq i$.
	Because $\hat\theta_i'=\infty$ we know that $\bar u_{i,\hat\theta_i'}=u_{i,\hat\theta_i'}=0$.
	Subsequently, $\psi_\lambda(\hat{\vct\theta}')=\psi_\lambda(\hat{\vct\theta}) - (\bar\phi_{i,\hat\theta_i}-\lambda)\bar u_{i,\hat\theta_i}>\psi_\lambda(\hat{\vct\theta}) = \lambda$.
	This contradicts the condition that $\psi_\lambda(\vct\theta)\leq\lambda$
	for all $\vct\theta\in\mathcal K_1\times\cdots\times\mathcal K_M$.
	
	Define $\psi_\lambda^0(\vct\theta) := \sum_{i=1}^M(\phi_{i,\theta_i}-\lambda)u_{i,\theta_i}$, which is similar to the definition of $\psi_\lambda$ except all occurrences of $\bar\phi_{i,\cdot}$
	and $\bar u_{i,\cdot}$ are replaced by their true values $\phi_{i,\cdot},u_{i,\cdot}$.
	Because $\bar\phi_{i,\hat\theta_i}\geq \bar R'(\hat{\vct\theta})$ for all $\bar u_{i,\theta_i}>0$,
	and $\bar\phi_{i\cdot},\bar u_{i\cdot}$ are upper bounds of $\phi_{i\cdot},u_{i\cdot}$,
	we conclude that $\psi_\lambda^0(\hat{\vct\theta}) \leq \psi_\lambda(\hat{\vct\theta})=\lambda$,
	implying that $\sum_{i=1}^M(\phi_{i,\hat\theta_i}-\lambda)u_{i,\hat\theta_i} \leq \lambda$.
	Re-arranging terms we have $R'(\hat{\vct\theta}) = [\sum_{i=1}^M\phi_{i,\hat\theta_i}u_{i,\hat\theta_i}]/[1+\sum_{i=1}^Mu_{i,\hat\theta_i}]\leq \lambda = \bar R'(\hat{\vct\theta})$.
	
	We next prove $\bar R'(\vct\theta^*)\geq R'(\vct\theta^*)$.
	Recall that $R'(\vct\theta^*)=[\sum_{i=1}^M\phi_{i,\theta_i^*}u_{i,\theta_i^*}]/[1+\sum_{i=1}^Mu_{i,\theta_i^*}]$.
	Hence, $\psi_\lambda^0(\vct\theta^*)=\lambda$ for $\lambda = R'(\vct\theta^*)$,
	meaning that $\sum_{i=1}^M(\phi_{i,\theta_i^*}-\lambda)u_{i,\theta_i^*}=\lambda$.
	By similar analysis, we know that $\phi_{i,\theta_i^*}\geq\lambda$ for all $u_{i,\theta_i^*}>0$ too.
	Because $\bar\phi_{i\cdot},\bar u_{i,\cdot}$ are upper bounds of $\phi_{i\cdot},u_{i\cdot}$ and $\bar u_{i,\theta_i^*}=0$ if $u_{i,\theta_i^*}=0$,
	we have $\psi_\lambda(\vct\theta^*) = \sum_{i=1}^M(\bar\phi_{i,\theta_i^*}-\lambda)\bar u_{i,\theta_i^*} \geq \sum_{i=1}^M(\phi_{i,\theta_i^*}-\lambda)u_{i,\theta_i^*}=\psi_\lambda^0(\vct\theta^*)=\lambda$.
	This implies that $\sum_{i=1}^M(\bar\phi_{i,\theta_i^*}-R'(\vct\theta^*))\bar u_{i,\theta_i^*}\geq R'(\vct\theta^*)$,
	and therefore $\bar R'(\vct\theta^*)=[\sum_{i=1}^M\bar\phi_{i,\theta_i^*}\bar u_{i,\theta_i^*}]/[1+\sum_{i=1}^M\bar u_{i,\theta_i^*}] \geq R'(\vct\theta^*)$.

\section{Proof of Lemma \ref{lem:aggregation}}	

	\begin{repeatlemma}[Lemma \ref{lem:aggregation} (restated).]
	With probability $1-T^{-1}$, for all $t\in[T]$, $i\in[M]$ and $\vct\theta=(\theta_1,\cdots,\theta_M)\in\mathcal K_1\times\cdots\times\mathcal K_M$,
	\begin{equation}
	\big|\bar R'(\vct\theta)-R'(\vct\theta)\big|\leq
	 \frac{1}{1+\sum_{i=1}^Mu_{i,\theta_i}}\left[\sum_{i=1}^M\frac{\bar u_{i,\theta_i}-u_{i,\theta_i}}{1+u_{i,\theta_i}} + \sum_{i=1}^M u_{i,\theta_i}(\bar\phi_{i,\theta_i}-\phi_{i,\theta_i})\right].
	\end{equation}
	\end{repeatlemma}
	\proof{Proof of Lemma \ref{lem:aggregation}}
	To simplify notations, we shall abbreviation $\phi_i=\phi_{i,\theta_i},u_i=u_{i,\theta_i}$ and $\bar\phi_i=\bar\phi_{i,\theta_i},\bar u_i=\bar u_{i,\theta_i}$.
	We also abbreviate $R'=R'(\vct\theta)$ and $\bar R'=\bar R'(\vct\theta)$.
	
	By definition of $R'$ and $\bar R'$, we have
	\begin{align}
	\left(1+\sum_{i=1}^Mu_i\right)[\bar R'-R']
	&=\left(1+\sum_{i=1}^Mu_i\right)\left[\frac{\sum_{i=1}^M \bar u_i\bar\phi_i}{1+\sum_{i=1}^M\bar u_i} -  \frac{\sum_{i=1}^M  u_i\phi_i}{1+\sum_{i=1}^M u_i}\right]\nonumber\\
	&= \sum_{i=1}^M\bar\phi_i\left(\bar u_i\frac{1+\sum_{i'=1}^Mu_{i'}}{1+\sum_{i'=1}^M\bar u_{i'}}-u_i\right) + \sum_{i=1}^M u_i(\bar\phi_i-\phi_i)\nonumber\\
	&{\leq  \sum_{i=1}^M\max\left\{0,\left(\bar u_i\frac{1+\sum_{i'=1}^Mu_{i'}}{1+\sum_{i'=1}^M\bar u_{i'}}-u_i\right)\right\}+ \sum_{i=1}^M u_i(\bar\phi_i-\phi_i)}\label{eq:aggregation-key}
	\end{align}
	
	The first term on the right-hand side of Eq.~(\ref{eq:aggregation-key}) can be further upper bounded by
	\begin{align}
	 \sum_{i=1}^M\bar u_i\frac{1+\sum_{i'=1}^Mu_{i'}}{1+\sum_{i'=1}^M\bar u_{i'}}-u_i
	 & = \frac{\sum_{i=1}^M\bar u_i(1+\sum_{i'=1}^Mu_{i'}) - \sum_{i=1}^Mu_i(1+\sum_{i'=1}^M u_{i'})}{1+\sum_{i=1}^M \bar u_i}
	 \leq \sum_{i=1}^M\frac{\bar u_i- u_i}{1+u_i}.
	\end{align}
	Here the last inequalities holds because the $\sum_{i,i'=1}^Mu_i\bar u_{i'}$ term cancels out, and $1+\sum_{i'=1}^M\bar u_{i'}\geq 1+\bar u_i\geq 1+u_i$.
	
\section{Proof of Lemma \ref{lem:regret-final-part1}}
	\begin{repeatlemma}[Lemma \ref{lem:regret-final-part1} (restate).]
	Conditioned on event $\mathcal A$, it holds that
	\begin{equation}
	\sum_{\tau}\sum_{i=1}^M\frac{\bar u_{i,\hat \theta_i^{(\tau)}}-u_{i,\hat \theta_i^{(\tau)}}}{1+u_{i,\hat \theta_i^{(\tau)}}}
	\lesssim \sqrt{MKT\log(MTK)} + MKU\log^2(MTK).
	\end{equation}
	\end{repeatlemma}
	\proof{Proof of Lemma \ref{lem:regret-final-part1}}
	
	{Define $T_\tau(i,\theta)$ as the size of $\mathcal T(i,\theta)$ after epoch $\tau$,
	and $T_0(i,\theta)$ as the final size of $\mathcal T(i,\theta)$ when Algorithm \ref{alg:ucb} terminates (i.e., {the total number of epochs in which assortment $\mathcal L(\theta)\in\mathcal K_i$ was offered in nest $i$}).
	Define also $\hat u_{i,\theta}^{(\tau)}$ to be the estimate of $u_{i,\hat \theta_i^{(\tau)}}$ at epoch $\tau$.
	Using Lemma \ref{lem:new-nested-concentration}, the expectation of the first term in Eq.~(\ref{eq:regret-intermediate-26}) can be upper bounded by,
	}
	\begin{align}
	&\sum_\tau\sum_{i=1}^M U\cdot \vct 1\{T_\tau(i,\hat\theta_i^{(\tau)})<3888\ln(2MTK)\} +  \left[\sqrt{\frac{3888\max(\hat u_{i,\theta}^{(\tau)},[\hat u_{i,\theta}^{(\tau)}]^2)\ln(2MTK)}{T_\tau(i,\hat\theta_i^{(\tau)})\cdot (1+u_{i,\theta}^{(\tau)})^2}}\right.\nonumber\\
	&\;\;\;\;\left. + \frac{144\ln(2MTK)}{T_\tau(i,\hat\theta_i^{(\tau)})}\right]\cdot \vct 1\{T_\tau(i,\hat\theta_i^{(\tau)})\geq 3888\ln(2MTK)\}\nonumber\\
	&=\sum_{i=1}^M\sum_{\theta\in\mathcal K_i}\sum_{\ell=0}^{T_0(i,\theta)} U\cdot\vct 1\{\ell<3888\ln(2MTK)\} +
	 \left[\sqrt{\frac{3888\max(\hat u_{i,\theta}^{(\tau)},[\hat u_{i,\theta}^{(\tau)}]^2)\ln(2MTK)}{\ell(1+u_{i,\theta})^2}}\right.\nonumber\\
	 &\;\;\;\;\left.+\frac{144\ln(2MTK)}{\ell}\right]\vct 1\{\ell\geq 3888\ln(2MTK)\}\nonumber\\
	& \leq{\sum_{i=1}^M\sum_{\theta\in\mathcal K_i}\sum_{\ell=0}^{T_0(i,\theta)} U\cdot\vct 1\{\ell<3888\ln(2MTK)\} +
	 \left[\sqrt{\frac{3888\max(2\hat u_{i,\theta}^{(\tau)},4u_{i,\theta}^2)\ln(2MTK)}{\ell(1+u_{i,\theta})^2}}\right.}\nonumber\\
	 &\;\;\;\;{\left.+\frac{144\ln(2MTK)}{\ell}\right]\vct 1\{\ell\geq 3888\ln(2MTK)\}}\label{eq:intermediate-3quarter}\\
	 &\lesssim {\sum_{i=1}^M\sum_{\theta\in\mathcal K_i}\sum_{\ell=0}^{T_0(i,\theta)} U\cdot\vct 1\{\ell<3888\ln(2MTK)\} +
	\left[\sqrt{\frac{\max\{u_{i,\theta},u_{i,\theta}^2\}\ln(MTK)}{\ell(1+u_{i,\theta})^2}}+ \sqrt{\frac{\ln(MTK)}{\ell(1+u_{i,\theta})^2}}\times \right.} \nonumber\\
	&\;\;\;\;{\left. \left(\sqrt{\frac{U^2\ln(MTK)}{\ell}}+\frac{\ln(MTK)}{\ell}\right) + \frac{\ln(MTK)}{\ell}\right]\vct 1\{\ell\geq 3888\ln(2MTK)\}}\label{eq:regret-intermediate-3half}\\
	&\lesssim {\sum_{i=1}^M\sum_{\theta\in\mathcal K_i} U\log(MTK) + \sqrt{u_{i,\theta}T_0(i,\theta)\log(MTK)} + U\log T_0(i,\theta)\log(MTK)}\label{eq:regret-intermediate-4} \\
	&\lesssim MKU\log^2(MKT) + \sum_{i=1}^M\sum_{\theta\in\mathcal K_i}\sqrt{u_{i,\theta}T_0(i,\theta)\log(MTK)}.\label{eq:regret-intermediate-41}
		\end{align}
	
	Here Eq.~(\ref{eq:intermediate-3quarter}) holds because of the following.
	If $\hat u_{i,\theta}^{(\tau)}\leq 2$ then the inequality clearly holds.
	If $\hat u_{i,\theta}^{(\tau)}\geq 2$ and $\ell \geq 3888\ln(2MTK)$, we argue that $u_{i,\theta}\geq 1$ (conditioned on the event $\mathcal A$).
	Assume otherwise that $u_{i,\theta}< 1$.
	By Lemma \ref{lem:new-nested-concentration} and the facts that $\ell\geq 3888\ln(2MTK)$, $\hat u_{i,\theta}^{(\tau)}\geq 1$, it holds that
	$\hat u_{i,\theta}^{(\tau)}\leq u_{i,\theta} +\frac{1}{3}\hat u_{i,\theta}^{(\tau)}+0.1$.
	Re-arranging the terms and noting we have $\hat u_{i,\theta}^{(\tau)}\leq \frac{3}{2}u_{i,\theta}+0.2\leq 1.7$, contradicting $\hat u_{i,\theta}^{(\tau)}\geq 2$.
	Hence $u_{i,\theta}\geq 1$ and in this case Eq.~(\ref{eq:intermediate-3quarter}) remains valid because $\hat u_{i,\theta}^{(\tau)}\in[0.5u_{i,\theta},2u_{i,\theta}]$.
	
	Here Eq.~(\ref{eq:regret-intermediate-3half}) holds by plugging in upper bounds on $|\hat u_{i,\theta}^{(\tau)}-u_{i,\theta}|$ (Lemma \ref{lem:new-nested-concentration}).
	Note that starting from Eq.~(\ref{eq:regret-intermediate-3half}), we drop all numerical constants
	and only report the asymptotic scalings of the terms, as evidenced in the $\lesssim$ notation in Eq.~(\ref{eq:regret-intermediate-3half}).
	Eq.~(\ref{eq:regret-intermediate-4}) holds because $\max\{a,a^2\}/(1+a)^2\lesssim a$,
	 $\sum_{\ell\leq T_0(i, \theta)}\ell^{-1/2}\lesssim \sqrt{T_0(i, \theta)}$ and$\sum_{\ell\leq T_0(i, \theta)}\ell^{-1}\lesssim \log T_0(i, \theta)$;
	 Eq.~(\ref{eq:regret-intermediate-41}) holds by replacing $\log T_0$ with $\log(MTK)$.

	Applying Cauchy-Schwarz inequality and the fact that $\mathbb E[\hat n_{i,\tau}]=u_{i,\theta_i}$, $\mathbb E|\mathcal E_\tau|=1+\sum_{i=1}^M\mathbb E[\hat n_{i,\tau}]$, the summation term in \eqref{eq:regret-intermediate-41} can be further bounded by
	\begin{align}
	\sum_{i=1}^M\sum_{\theta\in\mathcal K_i}\sqrt{u_{i,\theta}T_0(i,\theta)\log(MTK)}
	&\leq \sqrt{M|\mathcal K|}\cdot \sqrt{\sum_{i=1}^M\sum_{\theta\in\mathcal K_i} u_{i,\theta}T_0(i,\theta)\log(MTK)}\nonumber\\
	&= \sqrt{M|\mathcal K|}\cdot \sqrt{\sum_{i=1}^M\sum_{\theta\in\mathcal K_i}\sum_{\tau\in\mathcal T(i,\theta)}\mathbb E[\hat n_{i,\tau}]\log(MTK)}\nonumber\\
	&\leq \sqrt{M|\mathcal K|}\cdot \sqrt{\sum_\tau \mathbb E[|\mathcal E_\tau|]\log(MTK)}\nonumber\\
	&= \sqrt{MKT\log(MTK)}.\nonumber
	\end{align}
	Subsequently,
	\begin{equation*}
	\sum_{\tau}\sum_{i=1}^M\frac{\bar u_{i,\hat \theta_i^{(\tau)}}-u_{i,\hat \theta_i^{(\tau)}}}{1+u_{i,\hat \theta_i^{(\tau)}}}
	\lesssim \sqrt{MKT\log(MTK)} + MKU\log^2(MTK).
	\end{equation*}
	
\section{Proof of Lemma \ref{lem:regret-final-part2}}
	\begin{repeatlemma}[Lemma \ref{lem:regret-final-part2} (restated).]
	Conditioned on event $\mathcal A$, it holds that
	\begin{equation}
	\sum_{\tau}\sum_{i=1}^Mu_{i,\hat \theta_i^{(\tau)}}(\bar\phi_{i,\hat \theta_i^{(\tau)}}-\phi_{i,\hat \theta_i^{(\tau)}})
	\lesssim \sqrt{MKT\log(MTK)} + MKU\log^2(MTK).
	\end{equation}
	\end{repeatlemma}
	\proof{Proof of Lemma \ref{lem:regret-final-part2}}
	We first state the following result is a corollary of Lemma \ref{lem:new-nested-concentration} which gives a lower bound (with high probability) on $T(i,\theta)\hat u_{i,\theta}$
	when $u_{i,\theta}$ is not too small.
	Its proof is given at the end of this section.
	\begin{corollary}
	With probability $1-T^{-1}$ for all $i\in[M]$, $\theta\in\mathcal K_i$ such that $u_{i,\theta}\geq 768\ln(2MTK)/T(i,\theta)$ and $T(i,\theta)\geq 96\ln(2MTK)$, we have
	$T(i,\theta)\hat u_{i,\theta}\geq 0.5T(i,\theta)u_{i,\theta}$.
	\label{cor:uik-chernoff}
	\end{corollary}
		\proof{Proof of Corollary \ref{cor:uik-chernoff}}
	First consider the case of $u_{i,\theta}\geq 1$.
	By Eq.~(\ref{eq:u-true-ci-part2}) in the proof of Lemma \ref{lem:new-nested-concentration}, if $T(i,\theta)\geq 96\ln(2MTK)$ we have $|\hat u_{i,\theta}-u_{i,\theta}|\leq 0.5u_{i,\theta}$ and therefore $T(i,\theta)\hat u_{i,\theta}\geq 0.5 T(i,\theta)u_{i,\theta}$.
	
	In the rest of the proof we consider the case of $768\ln(2MTK)/T(i,\theta)\leq u_{i,\theta}\leq 1$. By Eq.~(\ref{eq:u-true-ci-part1}) in the proof of Lemma \ref{lem:new-nested-concentration}, we have
	$$
	T(i,\theta)\hat u_{i,\theta} \geq T(i,\theta)u_{i,\theta} - \sqrt{48T(i,\theta)u_{i,\theta}\ln(2MTK)} - 48\ln(2MTK).
	$$
	
	Under the condition that $u_{i,\theta}\geq 768\ln(2MTK)/T(i,\theta)$, the above inequality yields $T(i,\theta)\hat u_{i,\theta}\geq 0.5T(i,\theta)u_{i,\theta}$.
	$\square$
	
	Combining Corollary \ref{cor:uik-chernoff} with Lemma \ref{lem:new-nested-concentration} and noting that $|\bar\phi_{i,\theta}-\phi_{i,\theta}|\leq 1$ always holds, the second term on the right-hand side
	of Eq.~(\ref{eq:regret-intermediate-26}) can be upper bounded by
	\begin{align}
	&\sum_{i=1}^M\sum_{\theta\in\mathcal K_i}\sum_{\ell=0}^{T_0(i,\theta)}
	U\vct 1\{\ell< 96\ln(2MTK)\} +\left[\frac{768\ln(2MTK)}{\ell} + u_{i,\theta}\sqrt{\frac{2\ln(2MTK)}{T(i,\theta)u_{i,\theta}}}\right]\cdot \vct 1\{\ell\geq 96\ln(2MTK)\}\nonumber\\
	&\lesssim \sum_{i=1}^M\sum_{\theta\in\mathcal K_i} U\log(MKT) + \sqrt{u_{i,\theta}T_0(i,\theta)\log(MKT)} + \log T_0(i,\theta)\log(MKT).
	\end{align}
	
	Using similar derivation as in Eq.~(\ref{eq:regret-final-part1}), we have
	\begin{equation*}
	\sum_{\tau}\sum_{i=1}^Mu_{i,\hat \theta_i^{(\tau)}}(\bar\phi_{i,\hat \theta_i^{(\tau)}}-\phi_{i,\hat \theta_i^{(\tau)}})
	\lesssim \sqrt{MKT\log(MTK)} + MKU\log^2(MTK).
	\end{equation*}
	

\section{Proof of Lemma \ref{lem:S-diff}}
\begin{repeatlemma}[Lemma \ref{lem:S-diff} (restated).]
Let $U\subseteq[M]$ be the set of Type A nests, and by construction $[M]\backslash U$ are all Type B nests.
For any $\mat S=(S_1,\cdots,S_M)\in [N]^M$, define $m_U^\sharp(\mat S) := \sum_{i\in U}\vct 1\{S_i\neq \{1,2\}\} + \sum_{i\notin U}\vct 1\{S_i\neq \{1,2,3\}\}$.
Then there exists a numerical constant $C>0$ such that for all $\mat S$, $R(\mat S^*)-R(\mat S) \geq m_U^\sharp(\mat S)\cdot C\epsilon/M$,
where $\mat S^*\in\arg\max_{\mat S}R(\mat S)$ is the optimal assortment combination under $U$.
\end{repeatlemma}
\proof{Proof of Lemma \ref{lem:S-diff}}
For any $U\subseteq[M]$, $S\subseteq\{1,2,3\}$ and $\mat S=(S_1,\cdots,S_M)\in[3]^M$,
define $m_{U,S}^\sharp(\mat S) := \sum_{i\in U}\vct 1\{S_i = S\}$ and similarly $m_{U^c,S}^\sharp(\mat S) := \sum_{i\notin U}\vct 1\{S_i=S\}$.
Denote also $\mat S^*=(S_1^*,\cdots,S_M^*)$ as the optimal assortment combination, in which $S_i=\{1,2\}$ for all $i\in U$ and $S_i=\{1,2,3\}$ for all $i\notin U$.
Let also $R_U(\cdot)$, $V_U(\cdot)$, $R_{U^c}(\cdot)$, $V_{U^c}(\cdot)$
be revenue and preference of assortment selections in nests of Type A ($R_U(\cdot)$ and $V_U(\cdot)$) or Type B ($R_{U^c}(\cdot)$ and $V_{U^c}(\cdot)$), respectively.
Recall that $|U|=M/4$ and $|U^c|=3M/4$. We then have
\begin{align}
R(\mat S^*)-R(\mat S)
&= \frac{R_U(\{1,2\})V_U(\{1,2\})^{1/2}\cdot M/4 + R_{U^c}(\{1,2,3\})V_{U^c}(\{1,2,3\})^{1/2}\cdot 3M/4}{1 + V_U(\{1,2\})^{1/2}\cdot M/4 + V_{U^c}(\{1,2,3\})^{1/2}\cdot 3M/4}\nonumber\\
&\;\;\ - \frac{\sum_{S\subseteq\{1,2,3\}}m_{U,S}^\sharp(\mat S)\cdot R_U(S)V_U(S)^{1/2} + m_{U^c,S}^\sharp(\mat S)\cdot R_{U^c}(S)V_{U^c}(S)^{1/2}}{1 +
\sum_{S\subseteq\{1,2,3\}} m_{U,S}^\sharp(\mat S)\cdot V_U(S)^{1/2} + m_{U^c,S}^\sharp(\mat S)\cdot V_{U^c}(S)^{1/2}}.
\end{align}

We next list the values of $V_U(\cdot), R_U(\cdot), V_{U^c}(\cdot)$ and $R_{U^c}(\cdot)$ under our adversarial construction, shown in Table \ref{tab:construction}.
\begin{eqnarray*}
S = \emptyset:&& V_U(S)^{1/2} = 0,\;\; R_U(S) = 0,\;\; V_{U^c}(S)^{1/2} = 0, \;\;R_{U^c}(S) = 0;\\
S = \{1\}:&& V_U(S)^{1/2}= \frac{\sqrt{1+\epsilon}}{M},\;\; R_U(S) = 1, \;\; V_{U^c}(S)^{1/2} = \frac{\sqrt{1-\epsilon}}{M},\;\; R_{U^c}(S) = 1;\\
S = \{2\}:&& V_U(S)^{1/2} = \frac{\sqrt{1-\epsilon}}{M},\;\; R_U(S) = .8,\;\; V_{U^c}(S)^{1/2} = \frac{\sqrt{1+\epsilon}}{M}, \;\; R_{U^c}(S) = .8;\\
S=\{3\}: && V_U(S)^{1/2} = \frac{1}{M},\;\; R_U(S) = \rho,\;\; V_{U^c}(S)^{1/2} = \frac{1}{M},\;\; R_{U^c}(S) = \rho;\\
S = \{1,2\}:&& V_U(S)^{1/2}  = \frac{\sqrt{2}}{M},\;\; R_U(S) = .9+.1\epsilon,\;\; V_{U^c}(S)^{1/2} = \frac{\sqrt{2}}{M}, R_{U^c}(S) = .9-.1\epsilon;\\
S=\{1,3\}:&& V_U(S)^{1/2} = \frac{\sqrt{1+\epsilon}}{M},\;\; R_U(S) = \frac{1+\rho+\epsilon}{2+\epsilon},\;\; V_{U^c}(S)^{1/2} = \frac{\sqrt{1-\epsilon }}{M},\;\; R_{U^c}(S) = \frac{1+\rho-\epsilon}{2-\epsilon};\\
S=\{2,3\}:&& V_U(S)^{1/2} = \frac{\sqrt{1-\epsilon}}{M},\;\;R_U(S) = \frac{.8+\rho-.8\epsilon}{2-\epsilon},\;\;V_{U^c}(S)^{1/2} = \frac{\sqrt{1+\epsilon }}{M},\;\; R_{U^c}(S) =\frac{.8+\rho+.8\epsilon}{2+\epsilon};\\
S=\{1,2,3\}:&& V_U(S)^{1/2} = \frac{\sqrt{3}}{M}, \;\; R_U(S) = \frac{1.8+\rho+.2\epsilon}{3},\;\; V_{U^c}(S)^{1/2} = \frac{\sqrt{3}}{M},\;\; R_{U^c}(S) = \frac{1.8+\rho-.2\epsilon}{3}.
\end{eqnarray*}
	
Plugging the values of $V_U(\cdot),R_U(\cdot),V_{U^c}(\cdot),R_{U^c}(\cdot)$ into $R(\mat S^*)-R(\mat S)$, and taking $\epsilon\to 0^+$,
by detailed algebraic calculations we proved the lemma.
%
	
\section{Proof of Lemma \ref{lemma:KL}}
\begin{repeatlemma}[Lemma \ref{lemma:KL} (restated).]
Suppose $|U\triangle W|=1$, where $U\triangle W = (U\backslash W)\cup(W\backslash U)$ denotes the symmetric difference between subsets $U,W\subseteq[M]$.
Then there exists a constant $C'>0$ such that for any $\mat S=(S_1,\cdots,S_M)$, $\min\{\kl(P_U(\cdot|\mat S)\|P_{W}(\cdot|\mat S),\kl(P_W(\cdot|\mat S)\|P_U(\cdot|\mat S))\}) \leq C'\epsilon^2/M$.
\end{repeatlemma}
\proof{Proof of Lemma \ref{lemma:KL}}
By symmetry we may assume without loss of generality that $W=U\cup\{i_0\}$ for some $i_0\notin U$.
The random variables observable are $(i,j)$ where $i\in[M]\cup\{0\}$ indicates the nest in which a purchase is made (if no purchase is made then $i=0$)
and $j\in[N]=\{1,2,3\}$ is the particular item purchased in nest $i$ (if $i=0$ simply define $j=0$ with probability 1).
The KL divergence $\kl(P_U(\cdot|\mat S)\|P_W(\cdot|\mat S))$ can then be written as
\begin{align}
\kl(P_U(\cdot|\mat S)\|P_W(\cdot|\mat S))
&= -\mathbb E_U\left[\log\frac{P_W(i,j|\mat S)}{P_U(i,j|\mat S)}\right]
= -\mathbb E_U\left[\log\frac{P_W(i|\mat S)}{P_U(i|\mat S)}\right] - \mathbb E_U\left[\log\frac{P_W(j|i,\mat S)}{P_U(j|i,\mat S)}\right].
\label{eq:cond-kl}
\end{align}

We next upper bound the first term on the right-hand side of Eq.~(\ref{eq:cond-kl}).
By the nested model, the nest-level purchase action $i\in[M]\cup\{0\}$ follows a categorical distribution of $M+1$ categories,
parameterized by probabilities $\vct p=(p_0,\cdots,p_M)$ under $U$ and $\vct q=(q_0,\cdots,q_M)$ under $W$.
By elementary algebra (see for example Lemma 3 in \citep{Chen:18tight}), $\kl(\vct p\|\vct q)$ can be upper bounded as
\begin{equation*}
\kl(\vct p\|\vct q) = -\sum_{i=0}^M p_i\log\frac{q_i}{p_i} \leq \sum_{i=0}^M\frac{|p_i-q_i|^2}{q_i}.
\end{equation*}

Note that $U$ and $W$ only differ in nest $i_0$.
Using the nested model description and $\gamma_i\equiv 0.5$, it is easy to verify that $|p_i-q_i|\lesssim \epsilon/M$
for $i\in\{0,i_0\}$, $|p_i-q_i|\lesssim \epsilon/M^2$ if $i\notin\{0,i_0\}$, $q_0\gtrsim \Omega(1)$ and $q_i\gtrsim 1/M$ for all $i\geq 1$.
Subsequently,
\begin{equation}
\kl(\vct p\|\vct q) \lesssim \epsilon^2 / M. \label{eq:kl-term1}
\end{equation}

We proceed to upper bound the second term on the right-hand side of Eq.~(\ref{eq:cond-kl}).
Because $U$ and $W$ only differ in nest $i_0$, this term is non-zero only if $i=i_0$.
Conditioned on $i=i_0$, it is easy to verify that $\kl(P_U(\cdot|i_0,S_{i_0})\|P_W(\cdot|i_0,S_{i_0}))\lesssim \epsilon^2$
for all $S_{i_0}\subseteq[N]$.
In addition, $\max\{P_U(i_0|\mat S),P_W(i_0|\mat S)\}\lesssim 1/M$.
Subsequently,
\begin{equation}
- \mathbb E_U\left[\log\frac{P_W(j|i,\mat S)}{P_U(j|i,\mat S)}\right]
= P_U(i_0|\mat S)\cdot \kl(P_U(\cdot|i_0,S_{i_0})\|P_W(\cdot|i_0,S_{i_0})) \lesssim \epsilon^2/M.
\label{eq:kl-term2}
\end{equation}

Combining Eqs.~(\ref{eq:kl-term1},\ref{eq:kl-term2}) we complete the proof of Lemma \ref{lemma:KL}.

\section{Proof of Lemma \ref{lem:discretization}}
\begin{repeatlemma}[Lemma \ref{lem:discretization} (restated).]
	Fix an arbitrary $\delta\in(0,1)$. Then 
	$$
	\max_{\vct\theta\in\mathcal K_1\times\cdots\times\mathcal K_M} R'(\vct\theta) - \max_{\vct\theta\in\tilde{\mathcal K}_1^\delta\times\cdots\times\tilde{\mathcal K}_M^\delta} R'(\vct\theta) \leq\delta,
	$$
	where $R'(\vct\theta) := [\sum_{i=1}^M\phi_{i,\theta_i}u_{i,\theta_i}]/[1+\sum_{i=1}^Mu_{i,\theta_i}]$.
\end{repeatlemma}
\proof{Proof of Lemma \ref{lem:discretization}}

	Let $\vct{\theta}^*=(\theta_1^*,\cdots,\theta_M^*)\in\mathcal K_1\times\cdots\times\mathcal K_M$ be the assortment that maximizes $R'$.
	Define $\tilde\theta_i^* := \lfloor\theta_i^*/\delta\rfloor\cdot \delta$ for all $i\in[M]$ and $\tilde{\vct\theta}^* := (\tilde\theta_1^*,\cdots,\tilde\theta_M^*)$.
	It is easy to verify that $\tilde{\theta}^*\in\tilde{\mathcal K}_1^\delta\times\cdots\times\tilde{\mathcal K}_M^\delta$.
	Therefore, it suffices to prove that $R'(\tilde{\vct\theta}^*) \geq R^*-\delta$ where $R^*=R'(\vct\theta^*)$.
	
	To simplify notations, abbreviate $R_i=R_i(\mathcal L_i(\theta_i^*))$, $V_i=V_i(\mathcal L_i(\theta_i^*))$, $\tilde R_i = R_i(\mathcal L_i(\tilde\theta_i^*))$ and $\tilde V_i = V_i(\mathcal L_i(\tilde\theta_i^*))$,
	where $R_i(\cdot)$ and $V_i(\cdot)$ are defined in Eqs.~(\ref{eq:vi},\ref{eq:ri}).
	Denote also that $x_i := \tilde V_i-V_i$.
	By definition of $R_i$ and $\tilde R_i$, we have $R_iV_i = \sum_{r_{ij}\geq\theta_i^*}r_{ij}v_{ij}$ and $\tilde R_i\tilde V_i = \sum_{r_{ij}\geq\tilde\theta_i^*}r_{ij}v_{ij}$.
	Subsequently,
	\begin{equation}
	\tilde R_i\tilde V_i = R_iV_i + \sum_{\theta_i^*>r_{ij}\geq\tilde\theta_i^*}r_{ij}v_{ij} \geq R_iV_i + x_i(\theta_i^*-\delta).
	\label{eq:intermediate1-discretization}
	\end{equation}
	Here the last inequality holds because $|\theta_i^*-\tilde\theta^*|\leq\delta$ and $\sum_{\theta_i^*>r_{ij}\geq\tilde\theta_i^*}v_{ij} = \tilde V_i-V_i = x_i$.
	Subsequently,
	\begin{align}
	\tilde V_i^{\gamma_i}\left[\tilde R_i-(R^*-\delta)\right]
	&= (V_i+x_i)^{\gamma_i}\left[\tilde R_i-(R^*-\delta)\right]\\
	&\geq (V_i+x_i)^{\gamma_i}\left[\frac{R_iV_i+x_i(\theta_i^*-\delta)}{V_i+x_i} - R^*+\delta\right]\label{eq:intermediate2-discretization}\\
	&\geq (V_i+x_i)^{\gamma_i}\left[\frac{R_iV_i+x_i,\theta_i^*}{V_i+x_i}-R^*\right].\label{eq:intermediate3-discretization}
	\end{align}
	Here in Eq.~(\ref{eq:intermediate2-discretization}) we apply Eq.~(\ref{eq:intermediate1-discretization}),
	and Eq.~(\ref{eq:intermediate3-discretization}) holds because $x_i/(V_i+x_i)\leq 1$.
	
	\begin{proposition}
		For $i\in[M]$ define function $h_i(\Delta) := (V_i+\Delta)^{\gamma_i}[(R_iV_i+\Delta\theta_i^*)/(V_i+\Delta)-R^*]$.
		Then $h_i$ is monotonically non-decreasing in $\Delta$ for $\Delta\geq 0$.
		\label{prop:monotonicity}
	\end{proposition}
	
	Invoking Proposition \ref{prop:monotonicity}, we have that for all $i\in[M]$,
	\begin{equation}
	\tilde V_i^{\gamma_i}\left[\tilde R_i-(R^*-\delta)\right]
	\geq (V_i+x_i)^{\gamma_i}\left[\frac{R_iV_i+x_i,\theta_i^*}{V_i+x_i}-R^*\right] \geq V_i^{\gamma_i}\left[R_i-R^*\right].
	\end{equation}
	Summing over $i\in[M]$ on both sides of the above inequality and using the definition that $R^*=(\sum_{i\in[M]}R_iV_i^{\gamma_i})/(1+\sum_{i\in[M]}V_i^{\gamma_i})$,
	\begin{equation}
	\sum_{i\in[M]}\tilde V_i^{\gamma_i}\left[\tilde R_i-(R^*-\delta)\right]
	\geq \sum_{i\in[M]}R_iV_i^{\gamma_i} - \left(\sum_{i\in[M]}V_i^{\gamma_i}\right)R^* = R^* \geq R^*-\delta.
	\end{equation}
	Re-organizing terms we have
	$$
	R'(\tilde{\vct\theta}^*)=\frac{\sum_{i=1}^M\phi_{i,\tilde\theta_i^*}u_{i,\tilde\theta_i^*}}{1+\sum_{i=1}^Mu_{i,\tilde\theta_i^*}}
	= \frac{\sum_{i\in[M]} R_i(\mathcal L_i(\tilde\theta_i^*))V_i(\mathcal L_i(\tilde\theta_i^*))^{\gamma_i}}{1+\sum_{i\in[M]}V_i(\mathcal L_i(\tilde\theta_i^*))^{\gamma_i}}
	 = \frac{\sum_{i\in[M]}\tilde R_i\tilde V_i^{\gamma_i}}{1+\sum_{i\in[M]}\tilde V_i^{\gamma_i}} \geq R^*-\delta,
	$$
	which completes the proof.

	\begin{repeatproposition}[Proposition \ref{prop:monotonicity} (restated).]
	For $i\in[M]$ define function $h_i(\Delta) := (V_i+\Delta)^{\gamma_i}[(R_iV_i+\Delta\theta_i^*)/(V_i+\Delta)-R^*]$.
	Then $h_i$ is monotonically non-decreasing in $\Delta$ for $\Delta\geq 0$.
    \end{repeatproposition}
\proof{Proof of Proposition \ref{prop:monotonicity}}
	Note that $h_i(\Delta) = (V_i+\Delta)^{\gamma_i-1}(R_iV_i+\Delta\theta_i^*) - (V_i+\Delta)^{\gamma_i} R^*$.
	Differentiating $h_i$ with respect to $\Delta$ we have
	\begin{equation}
	h_i'(\Delta) = (\gamma_i-1)(V_i+\Delta)^{\gamma_i-2}(R_iV_i+\Delta\theta_i^*) + \theta_i^*(V_i+\Delta)^{\gamma_i-1} - \gamma_i(V_i+\Delta)^{\gamma_i-1}R^*.
	\end{equation}
	Using the second property of Lemma \ref{lem:nested-popular} that $\theta_i^* \geq \gamma_iR^*+(1-\gamma_i)R_i(S_i^*)$,
	we have for all $\Delta\geq 0$ that
	\begin{align}
	h_i'(\Delta)
	&\geq (\gamma_i-1)(V_i+\Delta)^{\gamma_i-2}(R_iV_i+\Delta\theta_i^*) + \left[\gamma_iR^*+(1-\gamma_i)R_i\right](V_i+\Delta)^{\gamma_i-1} \nonumber\\
	&\;\;\;\;- \gamma_i(V_i+\Delta)^{\gamma_i-1}R^*\\
	&= (1-\gamma_i)(V_i+\Delta)^{\gamma_i-2}\left[R_i(V_i+\Delta)-R_iV_i-\Delta\theta_i^*\right]\\
	&= (1-\gamma_i)(V_i+\Delta)^{\gamma_i-2}\cdot (R_i-\theta_i^*)\Delta \geq 0.
	\end{align}
	The lemma is then proved, because $h_i'(\Delta)\geq 0$ for all $\Delta\geq 0$.

\section{References to some concentration inequalities}

\begin{lemma}[Hoeffding's inequality \citep{hoeffding1963probability}]
	Suppose $X_1,\cdots,X_n$ are i.i.d.~random variables such that $a\leq X_i\leq b$ almost surely.
	Then for any $t>0$,
	$$
	\Pr\left[\left|\frac{1}{n}\sum_{i=1}^n{X_i} - \mathbb EX\right| > t\right] \leq 2\exp\left\{-\frac{2nt^2}{(b-a)^2}\right\}.
	$$
	\label{lem:hoeffding}
\end{lemma}

\begin{lemma}[Bernstein's inequality \citep{bernstein1924modification}]
	Suppose $X_1,\cdots,X_n$ are i.i.d.~random variables such that $\mathbb E[(X_i-\mathbb EX_i)^2]\leq\sigma^2$
	and $|X_i|\leq M$ almost surely. Then for any $t>0$,
	$$
	\Pr\left[\left|\frac{1}{n}\sum_{i=1}^n{X_i} - \mathbb EX\right|>t\right] \leq 2\exp\left\{-\frac{nt^2/2}{\sigma^2 + Mt/3}\right\}.
	$$
	\label{lem:bernstein}
\end{lemma}

The following result is cited from Theorem 5 of \citep{Agrawal16MNLBandit}.

\begin{lemma}[Concentration of geometric random variables \citep{Agrawal16MNLBandit}]
	Suppose $X_1,\cdots,X_n$ are i.i.d.~geometric random variables with parameters $p>0$,
	meaning that $\Pr[X_i=k]=(1-p)^kp$ for $k=0,1,2,\cdots$.
	Define $\mu := \mathbb EX_i = (1-p)/p$. Then
	\begin{equation*}
	\Pr\left[\frac{1}{n}\sum_{i=1}^nX_i > (1+\delta)\mu\right] \leq \left\{\begin{array}{ll}\exp\left\{-\frac{n\mu\delta^2}{2(1+\delta)(1+\mu)^2}\right\},& \text{if }\mu\leq 1,\\
	\exp\left\{-\frac{n\delta^2\mu^2}{6(1+\mu)^2}\left(3-\frac{2\delta\mu}{1+\mu}\right)\right\},& \text{if }\mu\geq 1, \delta\in(0,1);\end{array}\right.
	\end{equation*}
	\begin{equation*}
	\Pr\left[\frac{1}{n}\sum_{i=1}^n X_i < (1-\delta)\mu\right] \leq \left\{\begin{array}{ll}
	\exp\left\{-\frac{n\delta^2\mu}{6(1+\mu)^2}\left(3-\frac{2\delta\mu}{1+\mu}\right)\right\},& \text{if }\mu\leq 1,\\
	\exp\left\{-\frac{n\delta^2\mu^2}{2(1+\mu)^2}\right\},& \text{if }\mu \geq 1.\end{array}\right.
	\end{equation*}
	\label{lem:geometric}
\end{lemma}






\end{document}